\documentclass[english,11pt]{article}
 \usepackage{graphics}
 \usepackage{graphicx}
 \usepackage{epsfig}

\usepackage{amssymb}
\usepackage{amstext}
\usepackage{amsmath}
\usepackage{amsthm}
\usepackage{url}
\usepackage{color}

\newtheorem{example1}{Example} 
\newtheorem{definition}{Definition} 
\newtheorem{theorem}{Theorem} 
\newtheorem{proposition}{Proposition} 

\renewcommand{\O}{\mathcal{O}}

\newcommand{\naf}{\neg}

\newcommand{\KB}{\mathcal K\!\mathcal B}

\newcommand{\CTL}{\mathcal{CT}\!\!\mathcal L}

\newcommand{\per}{\mbox{\bf .}}

\newcommand{\ag}{\mathit{ag}}
\newcommand{\ex}{\mathit{ex}}
\newcommand{\eu}{\mathit{eu}}
\newcommand{\ef}{\mathit{ef}}

\newcommand{\eg}{\textit{eg}}

\newcommand{\tf}{t_{\!f}}

\newcommand{\ISA}{\sqsubseteq}

\newcommand{\AND}{\sqcap}

\newcommand{\SOME}[2]{\exists #1 \per #2}

\begin{document}




\title{Ontology-based Representation and Reasoning on Process Models: A Logic Programming Approach}

 \author{ Fabrizio Smith and Maurizio Proietti \\
National Research Council \\ 
Istituto di Analisi dei Sistemi ed Informatica "Antonio Ruberti" \\
Via dei Taurini 19, 00185, Roma, Italy\\
 \{fabrizio.smith, maurizio.proietti\}@iasi.cnr.it
}
\date{\vspace{-3ex}}
\maketitle

\begin{abstract}
We propose a framework grounded in Logic Programming for representing and reasoning about business processes from both the procedural and ontological point 
of views. In particular, our goal is threefold: (1) define a logical language and a formal semantics for process models enriched with ontology-based annotations; (2) provide an effective inference mechanism that supports the combination of reasoning services dealing with the structural definition of a process model, its behavior, and the domain knowledge related to the participating business entities; (3) implement such a theoretical framework into a process modeling and reasoning platform. 
To this end we define a process ontology coping with a relevant fragment of the popular  BPMN modeling notation. The behavioral semantics of a process is defined as a state transition system by following an approach similar to the Fluent Calculus, and allows us to specify state change in terms of preconditions and effects of the enactment of activities. 
Then we show how the procedural process knowledge can be seamlessly integrated with the domain knowledge specified by using the OWL 2 RL rule-based ontology language. 
Our framework provides a wide range of reasoning services, including CTL model checking, which can be performed by using standard Logic Programming inference engines through a goal-oriented, efficient, sound and complete evaluation procedure. We also present a software environment implementing the proposed framework, and we report on an experimental evaluation of the system, whose results are encouraging and show the viability of the approach.
\end{abstract}

\noindent \textbf {Keywords}: Business Processes, Ontologies, Logic Programming,
Knowledge Representation, Verification.


\section{Introduction}

The adoption of structured and systematic approaches for the management of
Business Processes (BPs) that operate within an organization is constantly gaining popularity,
especially in medium to large organizations such as manufacturing enterprises,
service providers, and public administrations.
The core of such approaches is the development of
BP models that represent the knowledge about processes in machine accessible form.
One of the main advantages of process modeling is that it enables automated analysis facilities,
such as the verification that the requirements specified over the models are enforced.
The automated analysis issue is addressed in the BP Management (BPM)
community mainly from a control flow perspective,
with the aim of verifying whether the behavior of the modeled system presents
logical errors (see, for instance, the notion of soundness \cite{vander-workflow}).

Unfortunately, standard BP modeling languages are not fully adequate to capture
process knowledge in all its aspects. While their focus is on the procedural representation
of a BP as a workflow graph that specifies the planned order of operations, the  domain
knowledge regarding the  entities involved in such a process, i.e., the business
environment in which processes are carried out, is often left implicit. This kind of
knowledge is typically expressed through natural language comments and labels
attached to the models, which constitute very limited, informal and ambiguous
pieces of information.
The lack of a formal representation of the domain knowledge
within process models
is widely recognized as an obstacle for the further automation of BPM
tools and methodologies that
effectively support process analysis, retrieval,
and reuse \cite{semantic-bpm}.



In order to overcome this limitation,
the application of well-established techniques stemming from the area of
Knowledge Representation in the domains of BP modeling
\cite{semantic-bpm,bpkb-iswc,gpo-thesis,beyond-soundness} and Web Services \cite{owl-s,wsmo-book} has been
shown to be a promising approach.
In particular, the use of computational ontologies is the most established approach for
representing in a machine processable way the
knowledge about the domain where business processes operate, providing formal
definitions for the basic entities involved in a process, such as activities, actors,
data items, and the relations between them. However, there are still several open issues
regarding the combination of BP modeling languages (with their execution
semantics) and ontologies, and the accomplishment of behavioral reasoning tasks
involving both these components. Indeed, most of the approaches developed for 
the semantic enrichment of process models or Web Services (such as the above cited
ones) do not provide an adequate model theory nor an axiomatization to capture and 
reasoning on dynamic aspects of process descriptions. 
On the other hand, approaches based on action languages developed in AI (e.g., \cite{golog-owl-s,swso,situation-calc-owl-s}) are very expressive formalisms that can be used to simultaneously capture the process and the domain knowledge, but they are too general to be applied to BP modeling, and must be suitably restricted not only  towards decidability of reasoning but also to reflect the peculiarities of processes. Indeed, action languages  provide a limited support for process definition, in terms of workflow constructs, and they lack a clear mapping from standard (ontology and process) modeling languages.

The main objective of this paper is to design a framework for representing and reasoning
about business process knowledge from both the procedural
and ontological point of views. 
To achieve this goal, we do {\em not} propose yet
another business process modeling language, but we provide a 
framework based on Logic Programming (LP) \cite{lloyd}
for reasoning about process-related knowledge expressed by means of 
de-facto standards for BP
modeling, like  BPMN \cite{bpmn}, and ontology definition, like OWL \cite{owl}.
We define a rule-based procedural semantics for a relevant
fragment of BPMN, by following
an approach inspired by the {\em Fluent Calculus} \cite{fluent-calculus-survey}, and we extend it in order to take into account OWL annotations
that describe preconditions and effects of activities and events occurring within a BP. In particular,
we integrate our procedural BP semantics with the  OWL 2 RL profile  thanks to a common grounding in LP.  OWL 2 RL is indeed a fragment of the OWL ontology language that has a suitable rule-based presentation, thus constituting an excellent compromise between expressivity and efficiency.

The contributions of this paper can be summarized as follows.

After presenting the preliminaries in Section \ref{section:Prelim}, we propose, in Section~\ref{section:BPS},  a revised and extended version of the  Business Process Abstract Language (BPAL) \cite{bpal-dexa2010,bpal-icaart}, a process ontology for modeling the procedural semantics of a BP
regarded as a workflow. To this end we introduce an axiomatization  to cope with a relevant fragment of the BPMN 2.0 specification,
allowing us to deal with a large class of process models.

We then propose, in Section~\ref{section:sem-annotation},
an approach for the semantic annotation of BP models,
where BP elements are described by using an OWL 2 RL
ontology.

In Section~\ref{section:temporal-reasoning}
we provide a general verification mechanism by integrating the
temporal logic CTL \cite{clarke} within our framework, in order to analyze properties of  the \textit{states} that the system can reach,  by taking into account both the control-flow and the semantic annotation.

In Section \ref{section:reasoning} we show how a repository of semantically enriched
BPs can be organized in a Business Process Knowledge Base (BPKB), which, due to the common representation of its components in LP, provides a uniform and formal
framework that enables logical inference. We then discuss how,
by using state-of-the-art LP systems,
we can perform some very sophisticated reasoning tasks, such as verification, querying
and trace compliance checking,  that combine both the procedural and the 
domain knowledge relative to a BP.

 In Section \ref{sect:proof} we provide the computational characterization of the reasoning services that can be performed on top of a BPKB, showing in particular that, for a large class of them, advanced resolution strategies (such as SLG-Resolution \cite{tabling}) guarantee an efficient, sound and complete procedure.

In Section \ref{sect:implementation} we describe the implemented tool, which provides a
graphical user interface to support the semantic BP design,
and a reasoner, developed in XSB Prolog \cite{xsb12}, able to operate on the BPKB.  We also report on an evaluation of the system performance, demonstrating that complex reasoning tasks can be
performed on business process of small-to-medium size in an acceptable amount of
time and memory resources.

In Section \ref{sect:rel_work} we compare our work to related approaches
and in the concluding section we give a critical discussion of our approach, 
along with directions for future work. 
\section{Preliminaries}
\label{section:Prelim}

\noindent 
In order to clarify the terminology and the notation used throughout this paper, in this section we recall  some 
basic  notions  related to the BPMN notation \cite{bpmn},  
Description  Logics  \cite{DLhandbook} as well as foundations of the  OWL  \cite{owl} standard, and
Logic  Programming  \cite{lloyd}.

\subsection{BPMN}
Business Process Modeling and Notation (BPMN) \cite{bpmn} is a graphical language for BP modeling, standardized by the OMG (\url{http://www.omg.org}). The primary goal of BPMN is to provide a standard notation readily understandable by all business stakeholders, which  include the business analysts who create and refine the processes, the technical developers responsible for their implementation, and the business managers who monitor and manage them. 

A BPMN model is defined through a Business Process Diagram (BPD), which is a kind of flowchart incorporating constructs to represents the control flow, data flow, resource allocation (i.e., how the work is assigned to the participants), and exception handling (i.e., how erroneous behavior can be handled and compensated). We will briefly overview the core BPMN constructs referring to the example in Figure \ref{fig:example}. 

The constructs of BPMN are classified as flow objects, artifacts, connecting objects, and swimlanes.  

\textit{Flow objects} are partitioned into \textit{activities} (represented as rounded rectangles), \textit{events} (represented as circles), and \textit{gateways} (represented as diamonds).  Activities are a generic way of representing some kind of work performed within the process, and can be \textit{tasks} (i.e., atomic activities such as \textit{create\_order}) or \textit{compound activities} corresponding to the execution of entire sub-processes (e.g., \textit{create\_order}). Events denote something that ``happens'' during the enactment of a business process, and are classified as  \textit{start events}, \textit{intermediate events}, and \textit{end events} which can start (e.g., $s$), suspend (e.g., $ex$), or end (e.g., $e$) the process enactment. An intermediate event, such as \textit{ex}, attached to the boundary of an activity models exception handling.  Gateways model the branching and merging of the control flow. There are several types of gateways in BPMN, each of which may be used as a \textit{branch} gateway if it has multiple outgoing flows, or  a \textit{merge} gateway if it has multiple incoming flows. The split and join behavior depends on the semantics associated to each type of gateway. Exclusive branch gateways (e.g., \textit{g1}) are decision points where exactly one of a set of mutually exclusive alternative flows is selected, while an exclusive merge gateway (e.g., \textit{g2}) merges two incoming flows into a single one.  Parallel branch gateways (e.g., \textit{g7}) create parallel threads of execution, while parallel merge gateways (e.g., \textit{g8}) synchronize concurrent flows. Inclusive branch gateways (e.g., \textit{g3}) are decision points where at least one of a set of non-exclusive alternative flows is selected, while an inclusive merge gateway (e.g., \textit{g4}) is supposed to be able to synchronize a varying number of threads, i.e., it is executed only when at least one of its predecessors has been executed and no other will be eventually executed\footnote{For sake of completeness, BPMN provides two more types of gateways, which we do not exemplify, namely, the event-based and  the complex gateway.}.

\textit{Connecting objects} are \textit{sequence flows} (e.g., the directed edge between \textit{g1} and \textit{g3}) and \textit{associations} (e.g., the dashed edge between \textit{create\_order} and \textit{order}). A sequence flow links two flow objects and   denotes a control flow relation, i.e., it states that the control flow can pass from the source to the target object. An \textit{association} is used to associate \textit{artifacts} (i.e., data objects) with flow objects,  and its direction defines if a data object is used as an input  (e.g., \textit{order} is an input of \textit{accept\_order})  or it is an output (e.g., \textit{order} is an output of \textit{create\_order}) of some flow element.

\textit{Swimlanes} are used to model  participants, i.e., a generic notion representing a role within a company (e.g., Sales Clerk), a department (e.g., Finance) or
a business partner (e.g., Courier), which is assigned to the execution of a collection of activities.
\begin{center}
\begin{figure} [h]
\centering
  \includegraphics[width=15cm]{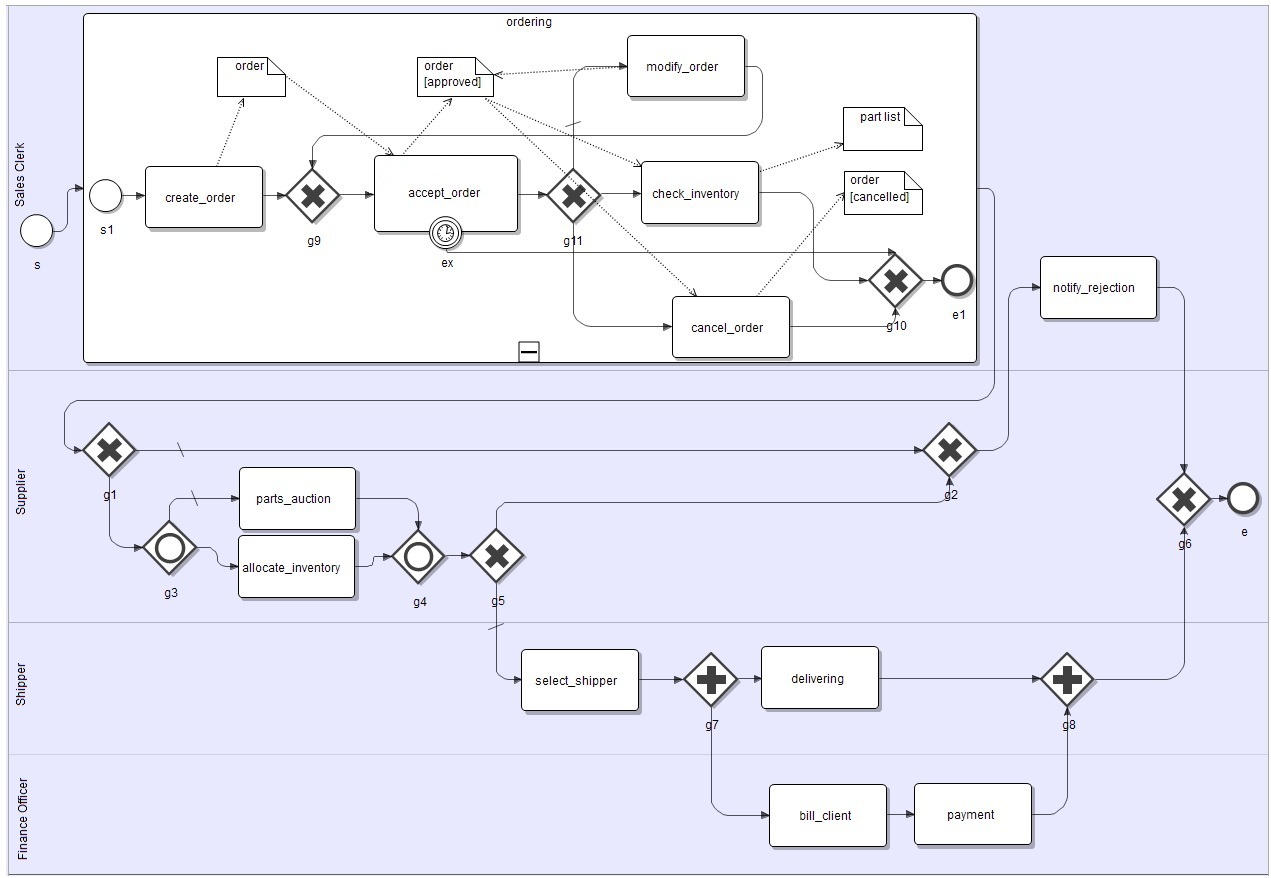}
	\caption{Handle Order Business Process}
	\label{fig:example}
\end{figure}	
\end{center}

\subsection{Description Logics and Rule-based OWL Ontologies}
\label{sect:owl}
Description Logics (DLs) \cite{DLhandbook}  are a family of knowledge representation languages that can be used to represent the knowledge of an application domain in a structured and formally well-understood way. DLs are typically adopted for the definition of ontologies since on the one hand, the important notions of the domain are described by  \textit{concept descriptions}, i.e., expressions that are built from atomic concepts (usually thought as sets of individuals, e.g., $Person$) and atomic roles (relations between concepts, e.g., $worksFor$) using the concept and role constructors provided by the particular DL (e.g., $ Person \AND \SOME{worksFor}{Company}$, that is, the set of \textit{persons} who work for a \textit{company}). On the other hand, DLs correspond to decidable fragments of classical first-order logic (FOL), and thus are equipped with a formal, logic-based semantics that makes such languages suitable for automated reasoning. 

\begin{table}[htbp]
\begin{center}
\caption{Main OWL statements and FOL equivalence}
\label{tab:owl2fol}
{\small 
\begin{tabular}{|l|c|c|} \hline 
\textbf{\textit{OWL Axiom}} & \textbf{\textit{DL Expression}} & \textbf{\textit{FOL Formula}}  \\ \hline 
\textit{a type C} & $a : C$ & $C(a)$  \\ \hline 
\textit{a P b} & $(a,b) : P$ & $P(a,b)$ \\ \hline 
\textit{C subClassOf D} & \textit{C }$\sqsubseteq \ $\textit{D} & $\forall \textit{x.}C(x) \to D(x)$  \\ \hline 
\textit{C disjointWith D} & \textit{C }$\sqsubseteq \neg$\textit{D} & $\forall $\textit{x.}$C(x) \to \neg D(x)$  \\ \hline 
\textit{P domain C} & $\top$ \textit{ }$\sqsubseteq $\textit{ }$\forall $\textit{P${}^{-}$.C} & $\forall \textit{x,y.}P(x,y) \to C(x)$  \\ \hline 
\textit{P range C} & $\top$ \textit{ }$\sqsubseteq $\textit{ }$\forall $\textit{P.C} & $\forall \textit{x,y.}P(x,y) \to C(y)$  \\ \hline 
\textit{transitiveProperty P} & \textit{P${}^{+}$ }$\sqsubseteq \ P$\textit{} & $\forall \textit{x,y,z.}(P(x,y) \wedge P(y,z)) \to P(x,z)$  \\ \hline 
\textit{functionalProperty P} & $\top$ \textit{ }$\sqsubseteq $\textit{ }$\le $1\textit{ P} & $\forall \textit{x,y,z.}(P(x,y) \wedge P(x,z)) \to y=z$  \\ \hline 
\textit{P inverseOf Q} & \textit{P}$\ \equiv $\textit{ Q${}^{-}$} & $\forall \textit{x,y.}P(x,y) \leftrightarrow Q(y,x)$  \\ \hline 
\textbf{\textit{OWL Constructor}}\textit{} & \multicolumn{2}{|p{2.8in}|}{\textit{}}  \\ \hline 
\textit{C intersectionOf D} & \textit{C }$\sqcap \ $\textit{D} & $C(x) \wedge D(x)$  \\ \hline 
\textit{C unionOf D} & \textit{C }$\sqcup \ $\textit{D} & $C(x) \vee D(x)$  \\ \hline 
\textit{P allValuesFrom C} & $\forall $\textit{P.C} & $\forall \textit{y.}P(x,y) \to C(y)$  \\ \hline 
\textit{P someValuesFrom C} & $\exists $\textit{P.C} & $\exists \textit{y.}P(x,y) \wedge C(y)$  \\ \hline 
\textit{complementOf D} & $\neg$\textit{D} & ${\rm \ }\neg D(x)$  \\ \hline 
\end{tabular}
}
\end{center}
\end{table}

Typically, Description Logics are used for representing a TBox (terminological box) and the ABox (Assertional Box). The
TBox describes concept (and role) hierarchies, (e.g., $Employee \ISA Person \AND \SOME{worksFor}{Company}$), while the ABox contains
assertions about individuals (e.g., $john  : Employee$). 

The growing interest in the Semantic Web vision \cite{semanticweb}, where Knowledge Representation techniques are adopted to  make resources machine-interpretable by ``intelligent agents'', has  pushed the standardization of languages  for ontology and meta-data sharing over the (semantic) web. Among these, one of the most promising standards is the Ontology Web Language (OWL) \cite{owl}, formally grounded in DLs, proposed by the Web Ontology Working Group of W3C. OWL is \textit{syntactically} layered on RDF \cite{rdf} and RDFS \cite{rdfs}, and can be considered as an extension of RDFS in terms of modeling capabilities and reasoning facilities. The underlying data model (derived from RDF) is based on  statements (or RDF triples) of the form $<\!subject, property, object\!>$, which allow us to describe a resource (\textit{subject}) in terms of named relations (\textit{properties}). Values of named relations (i.e. objects) can be URIrefs of Web resources or literals, i.e., representations of data values (such as integers and strings).

Table \ref{tab:owl2fol} shows, for some OWL statements,
the corresponding DL notations and FOL formulae, where \textit{C} and \textit{D} are concepts (OWL classes), \textit{P} and \textit{Q} are roles (OWL properties), \textit{a} and \textit{b} are constants, and \textit{x} and \textit{y} are variables. 

The recent OWL 2 specification defines  profiles that correspond to syntactic subsets of OWL, each of which is designed to trade some expressive power for efficiency of reasoning.  In particular, we consider OWL 2 RL, closely related to the Horn fragment of FOL, which is based on \textit{Description Logic Programs} \cite{dlp} and \textit{pD*} \cite{pd*}.
The use of OWL 2 RL allows us to take advantage of the 
efficient resolution strategies developed for logic programs, in order to perform the 
reasoning tasks typically supported  by Description Logics reasoning systems, such as 
concept subsumption and ontology consistency.
Indeed, the semantics of OWL 2 RL is defined through a partial axiomatization of the OWL 2 RDF-Based Semantics in the form of first-order implications (OWL 2 RL/RDF rules), and constitutes an upward-compatible extension of RDF and RDFS. 

OWL 2 RL  ontologies are modeled by means of the ternary predicate $t(s,p,o)$ 
representing an OWL statement  with subject \textit{s}, predicate \textit{p} and object 
\textit{o}. For instance, the assertion $t(\textit{a},\textit{rdfs:subClassOf},\textit{b})$
represents the inclusion axiom \textit{a} $\sqsubseteq $ \textit{b}. 
Reasoning on triples is supported by OWL 2 RL/RDF rules of the
form $t(s,p,o) \leftarrow t(s_1,p_1,o_1) \wedge \dots \wedge t(s_n,p_n,o_n)$. Table \ref{tab:owl-rl} shows some of the rules of the OWL 2 RL/RDF rule-set.
According to the terminology we will introduce in the next section, this rule set is a {\em definite logic program}.

\begin{table}[htbp!]
\begin{center}
\caption{Excerpt of the OWL 2 RL/RDF rule-set}
\label{tab:owl-rl}
\small{
\begin{tabular}{|l|l|}
\hline
Transitive  & $t(C_1,\textit{rdfs:subClassOf},C_3)\leftarrow t(C_1,\textit{rdfs:subClassOf},C_2)\wedge $\\
subsumption & \hspace*{5mm}$t(C_2,\textit{rdfs:subClassOf},C_3)$\\ 
\hline
Inheritance & $t(X,\textit{rdf:type},C_2)\leftarrow t(C_1,\textit{rdfs:subClassOf},C_2)\wedge t(X,\textit{rdf:type},C_1)$  \\

 & $t(X,\textit{rdf:type},C_2)\leftarrow t(C_1,\textit{owl:equivalentClass},C_2)\wedge t(X,\textit{rdf:type},C_2)$ \\
\hline
Domain & $t(X,\textit{rdf:type},C)\leftarrow t(P,\textit{rdfs:domain},C)\wedge t(X,P,O)$ \\
\hline
Range & $t(Y,\textit{rdf:type},C)\leftarrow t(P,\textit{rdfs:range},C)\wedge t(S,P,Y)$ \\
\hline
Transitivity & $t(X,P,Z)\leftarrow t(P,\textit{rdf:type},\textit{owl:TransitiveProperty})\wedge t(X,P,Y) \wedge t(Y,P,Z)$ \\
\hline
Subsumption & $t(C_1,\textit{rdfs:subClassOf},C_2) \leftarrow t(C_1,\textit{owl:someValuesFrom},D_1) \wedge $ \\ 
of existential & \hspace*{5mm}$t(C_1,\textit{owl:onProperty},P) \wedge t(C_2,\textit{owl:someValuesFrom},D_2)\wedge   $\\
 formulae & \hspace*{5mm}$t(C_2,\textit{owl:onProperty},P) \wedge t(D_1,\textit{rdfs:subClassOf},D_2)$  \\
\hline
Intersection & $t(C, \textit{rdfs:subClassOf}, D) \leftarrow t(C, \textit{owl:intersectionOf}, I) \wedge D \in I$  \\
\hline
Disjointness  &  $\bot \leftarrow t(C_1,\textit{owl:disjointWith},C_2) \wedge t(X,\textit{rdf:type},C_1) \wedge t(X,\textit{rdf:type},C_2)$ \\
\hline
\end{tabular}
}
\end{center}

\end{table}

\subsection{Logic programming}
 
\label{subsec:lp}

We briefly recall the basic notions of  Logic Programming.
In particular, we will consider the class of {\em locally stratified logic programs},
or {\em stratified programs}, for short,
and their standard semantics defined by the {\em perfect model}.
(Recall that all major declarative semantics of logic programs
coincide on stratified programs.)
This class of logic programs is expressive enough
to represent several complementary pieces of knowledge related to business processes,
such as the syntactic structure of the control flow, the operational semantics, 
the ontology-based properties, and the temporal properties of the execution.
For more details about LP we refer to \cite{lloyd,lp-negation}.

\medskip

A {\em term} is either a {\em constant}, or a {\em variable}, or 
an expression of the form $f(t_1,\ldots,t_m)$, where $f$ is a 
function symbol and $t_1,\ldots,t_m$ are terms.
An {\em atom} is a formula of the form $p(t_1,\ldots,t_m)$,
where $p$ is a {\em predicate symbol} and $t_1,\ldots,t_m$
are terms.
A {\em literal} is either an atom or a negated atom.
A {\em rule} is a formula of the form
$A \leftarrow L_1 \wedge \ldots \wedge L_n$, where $A$ is an {atom}
(the {\em head} of the rule)
and $L_1 \wedge \ldots \wedge L_n$ is a conjunction
of  literals (the {\em body} of the rule). 
If $n=0$ we call the rule a {\em fact}. 
A rule (term, atom, literal)  is {\em ground} if no variables occur in it.
A {\em logic program} is a set of rules. 
A {\em definite} program is a logic program
with no negated atoms in the body of its rules. 
For a logic program $P$, by $\textit{ground}(P)$ we denote 
the set of ground instances of rules in $P$.

Let $B_P$ denote the {\em Herbrand base} for $P$, that is, the set of ground
atoms that can be constructed in the language of program $P$.
An (Herbrand) {\em interpretation} $I$ is a subset of $B_P$.
A ground atom $A$ is true in $I$ if $A\in I$.
A ground negated atom $\neg A$ is true in $I$ if $A\not\in I$.
A ground rule $A \leftarrow L_1 \wedge \ldots \wedge L_n$
is true in $I$ if either $A$ is true in $I$
or, for some $i\in\{1,\ldots, n\}$, $L_i$ is not true in $I$.
An interpretation is a {\em model} of $P$ if all rules in $\textit{ground}(P)$
are true in $I$.
Every definite program has a {\em least} Herbrand model. However, this
property does not hold for general logic programs.

A {\em (local) stratification}
is a function $\sigma$ from the Herbrand base $B_P$ to the 
set of all countable ordinals~\cite{lp-negation,lp-Przymusinski}.
However, for the purposes of this paper it will be enough to consider
stratification functions from $B_P$ to the 
set $\mathbb{N}$ of the natural numbers.
For a ground atom $A$, $\sigma(A)$ is called the {\em stratum} of $A$.
A~stratification $\sigma$ extends to negated atoms by taking 
$\sigma(\neg A)\!=\!\sigma(A)+1$.
A ground rule $A \leftarrow L_1 \wedge \ldots \wedge L_n$ is 
{\em stratified with respect to} $\sigma$ if, for $i=1,\ldots,n,$
$\sigma(A)\!\geq\!\sigma(L_i)$.
A program $P$ is stratified with respect to $\sigma$ if 
every rule in $\textit{ground}(P)$ is.
Finally, a logic program is {\em stratified} if it is stratified 
with respect to some stratification function. 

The perfect model of $P$, denoted $\textit{Perf}(P)$, is
defined as follows.
Let $P$ be stratified with respect to $\sigma$. 
For every $n\in\mathbb{N}$, let $S_n$ be the set
of rules in $\textit{ground}(P)$ whose head has
stratum $n$.
Thus, $\textit{ground}(P)=\bigcup_{n\in \mathbb{N}} S_n$.
We define a sequence of interpretations as follows:
(i) $M_0$ is the least model of  $S_0$ 
(note that $S_0$ is a definite program), and
(ii) $M_{n+1}$ is the least 
model of $S_n$ that contains $M_n$.
The {perfect model} of $P$, 
is defined as $\textit{Perf}(P)=\bigcup_{n\in\mathbb{N}}M_{n}$.
(Here we are using the simplifying assumption that
the codomain of the stratification function is $\mathbb{N}$.)


The operational semantics of logic programs is based on 
the notion of derivation, which is constructed by
{\em SLD-resolution} augmented with the {\em Negation
as Failure} rule \cite{lloyd}.
Given a stratified 
program $P$, we will define below the {\em one-step derivation} relation 
$Q_1 \stackrel{\theta}{\longrightarrow}Q_2$,
where $Q_1, Q_2$ are {\em queries}, that is,
conjunctions of literals, and $\theta$ is a substitution.
The definition of one-step derivation relation
depends on the following notions.
A {\em derivation} for a query $Q_0$ with respect to 
$P$ is a sequence 
$Q_0 \stackrel{\theta_1}{\longrightarrow} \ldots \stackrel{\theta_n}{\longrightarrow} Q_n$ ($n\geq 1$).
We will omit the reference to $P$ when clear from the context.
A {derivation} is {\em successful} if its last query
is the empty conjunction \textit{true}.
A query {\em succeeds}  if there exists a
successful derivation for it.
A query {\em fails}   if it does not succeed.
The one-step derivation relation is defined by the following two derivation rules.

\pagebreak
\begin{enumerate}
\item[(P)] Let $A \wedge Q$ be a
query, where $A$ is an atom. Suppose that
$H \leftarrow  K_1 \wedge \ldots \wedge K_m$ ($m\geq 0$)
is a rule in $P$ such that $A$ is unifiable with $H$ via a most general unifier~$\theta$~\cite{lloyd}.
Then
$A \wedge Q \stackrel{\theta}{\longrightarrow} (K_1 \wedge \ldots \wedge K_m  \wedge Q)\theta$.

\item[(N)] Let $\neg A \wedge Q$ be a
query, where $A$ is a ground atom.
Suppose that $A$ fails.
Then $\neg A \wedge Q \stackrel{\epsilon}{\longrightarrow} Q$,
where $\epsilon$ is the identity substitution.
\end{enumerate}

Note that in the definition of a derivation 
we assume the left-to-right selection rule for literals.
Note also that, in rule (N)  the one-step derivation
from $\neg A \wedge Q$
refers to the set of all derivations from $A$
(to show that $A$ fails).
However, this definition is well-founded because
the program $P$ is stratified.
We say that a query $Q$ is {\em generable}
from a query $Q_0$ if there exists 
either a derivation $Q_0 \stackrel{\theta_1}{\longrightarrow} \ldots \stackrel{\theta_n}{\longrightarrow} Q$
or a derivation $Q_0 \stackrel{\theta_1}{\longrightarrow} \ldots \stackrel{\theta_n}{\longrightarrow} \neg A \wedge Q_n$ and 
$Q$ is generable from~$A$.
An {\em answer} for a query $Q_0$ 
is a substitution
$\theta$ such that there exists a successful 
derivation $Q_0 \stackrel{\theta_1}{\longrightarrow}  \ldots \stackrel{\theta_n}{\longrightarrow} \textit{true}$ and $\theta$ is
the restriction of the composition  $\theta_1\ldots \theta_n$ to the variables occurring in $Q_0$.
A query $Q_0$ {\em flounders} if there exists
a query $Q$ generable from $Q_0$ such that the leftmost literal
of $Q$ is a non ground negated atom.

The operational semantics is {\em sound} and {\em complete}
with respect to 
the perfect model semantics for queries that do not flounder.
Indeed, it can be shown that (see, for instance, \cite{lp-Przymusinski,lp-negation}), given a program $P$ and an atom $A_0$ that does not flounder with
respect to $P$, then: 
(1) if $A_0$ succeeds with answer 
$\theta$, then every ground instance of $A_0\theta$ belongs to 
$\textit{Perf}(P)$, and
(2) if $A_0\theta$ belongs to $\textit{Perf}(P)$ for some substitution $\theta$, 
then $A_0$ succeeds with an answer which is
more general than $\theta$.

The definition of a derivation given above is quite abstract
and not fully constructive. In particular, the
application of rule (N) requires to test that an atom has no successful derivations, 
and this property 
is undecidable in the general case.
Thus, an effective query evaluation strategy depends on
the concrete way derivations are constructed.

A well-known difficulty of the evaluation strategy
based on depth-first search 
is that infinite derivations may be constructed, 
even in cases where a finite set of atoms (modulo variants)
is derived from a given initial query.
In particular, this nonterminating behavior can occur
for stratified {\em Datalog} programs, that is, function
free stratified programs.


In order to avoid this difficulty, in this paper we adopt 
{\em SLG-resolution}, a query evaluation mechanism
that implements SLD resolution
with Negation as Failure by means of {\em tabling} \cite{tabling}.
During the construction of the derivations
for a given atom $A_0$, a {\em table} is maintained
to record the answers to $A_0$ and to
the atoms generated from $A_0$. 
The tabled answers are used
the next time an atom is generated,
and hence no atom is evaluated more than once.
Thus, SLG-resolution is able to compute in finite time
all answers to a query, if a finite set of atoms is
generated and a finite set of answers for those atoms exists.
In particular, SLG-resolution always terminates and is able to
compute all answers for queries to
stratified Datalog programs.


\section{Rule-based Representation of BP Schemas}
\label{section:BPS}

In this section we introduce a formal representation of business processes by
means of the notion of {\em Business Process Schema} (BPS).
A BPS, its meta-model, and its procedural (or {\em behavioral}) semantics 
will all be specified by sets of rules, for which
we adopt the standard notation and semantics of LP (see Section \ref{subsec:lp}).

\subsection{Introducing BPAL}
\label{section:bpal}

The Business Process Abstract Language (BPAL) introduces a language conceived to provide a declarative modeling method capable of fully capturing procedural knowledge in a business process. BPAL constructs are common to the most used and widely accepted BP modeling languages (e.g., BPMN \cite{bpmn}, UML activity diagrams \cite{uml-activity}, EPC \cite{vicious-circle}) and, in particular, it is based on the BPMN 2.0 specification \cite{bpmn}.

Formally, a (set of) BPS(s) $\mathcal B$ is specified by a set of {ground} {facts} of the 
form $p(c_1,\dots ,c_n)$, where $c_1,\dots ,c_n$ are constants denoting flow 
elements (e.g.,  activities, events, and gateways) and \textit{p }is a predicate
symbol. In Table \ref{tab:bpallan} we list some of the BPAL predicates, and in Table  \ref{tab:bpalex} we exemplify their usage reporting the  translation of the \textit{Handle Order} process ($ho$ for short) depicted in Figure \ref{fig:example} as a BPAL BPS. An extended discussion
can be found in \cite{bpal-dexa2010,bpal-bsme}.

\begin{table}[htbp]
\caption{Excerpt of the BPAL language }
\label{tab:bpallan}
\centering
{\small 
\begin{tabular}{|l|l|}
\hline
\textbf{Construct}								 &		\textbf{Description} 		\\
\hline
\textit{bp(p,s,e)} & \textit{p }is a process, with entry-point \textit{s }and exit-point \textit{e}  \\ \hline
\textit{element(x)} & \textit{x }is a flow \textit{element }occurring in some process  \\ \hline 
\textit{relation(x,y,p)} & the \textit{elements} \textit{x }and \textit{y }are in relation in the process \textit{p}   \\ \hline
\textit{task(a)} & \textit{a} is an atomic activity   \\ \hline
\textit{event(e)} & \textit{e} is an event   \\ \hline
\textit{exception(e,a,p)} & the intermediate event $e$ (an exception) is attached to the activity $a$ \\ \hline
\textit{comp\_act(a,s,e)} & \textit{a }is a compound activity, with entry-point \textit{s }and exit-point \textit{e}  \\ \hline 
\textit{seq(el1,el2,p)} & a sequence flow relation is defined between \textit{el1} and \textit{el2} in \textit{p}  \\ \hline
\textit{par\_branch(g}) & the execution of \textit{g }enables all the successor flow elements\textit{ }  \\ \hline
\textit{par\_merge(g)} & \textit{g }waits for the completion of all the predecessor flow elements\textit{ } \\ \hline
\textit{exc\_branch(g}) & the execution of \textit{g }enables one of the successor flow elements  \\ \hline
\textit{exc\_merge(g)} & \textit{g }waits for the completion of one of the predecessor flow elements  \\ \hline
\textit{inc\_branch(g}) & the execution of \textit{g }enables at least one of its successors   \\ \hline
\textit{inc\_merge(g)} & \textit{g }waits for the completion of the predecessor flow elements \\
& that will be  eventually executed  \\ \hline
\textit{item(i)} & \textit{i }is a data element \textit{}  \\  \hline
\textit{input(a,i,p)} & the activity \textit{a }uses as input the data element\textit{ i} in the process \textit{p}  \\ \hline
\textit{output(a,i,p)} & the activity\textit{ a }uses as output the data element\textit{ i} in the process \textit{p}  \\ \hline 
\textit{participant(part)} & \textit{part} is a participant \\ \hline
\textit{assigned(a,part,p)} & the activity \textit{a }is assigned to the participant \textit{part }in the process \textit{p} 	\\ \hline
\end{tabular}
}

\end{table}

\begin{table}[ht!]
\caption{BPS representing the Handle Order process}
\label{tab:bpalex}
\centering
{\footnotesize 
\textit{
\begin{tabular}{|l|l|l|}
\hline
bp(ho,s,e) & seq(s, ordering, ho)  & comp\_act(ordering, s$_1$, e$_1$) \\ \hline
seq(ordering,g1,ho) & seq(g1,g2, ho) & assigned(ordering,sales\_clerk,ho)    \\ \hline
seq(g1,g3,ho) & seq(g3,parts\_auction,ho) & assigned (delivering,shipper,ho)  \\ \hline 
seq(g3,allocate\_inventory,ho) & seq(parts\_auction,g4,ho) & seq(s1,create\_order,ordering)  \\ \hline 
seq(allocate\_inventory,g4,ho) & seq(g4,g5,ho) & seq(create\_order,g9,ordering) \\ \hline 
seq(g5,g2,ho)  & seq(g5,select\_shipper,ho)   & seq(g9,accept\_order,ordering)\\ \hline
seq(g2,notify\_rejection,ho)  & seq(select\_shipper,g7,ho) & exception(ex,accept\_order,ordering)\\ \hline
seq(notify\_rejection,g6,ho) & seq(g6,e,ho) & seq(ex,g10,ordering)\\ \hline
exc\_branch(g1) & participant(sales\_clerk) & input(accept\_order,order,ordering) \\ \hline 
inc\_branch(g3) &  task(create\_order)  & output(create\_order,order,ordering) \\ \hline
par\_branch(g7) & item(order)  & \ldots \\ \hline
\end{tabular}
}
}
\end{table}

\pagebreak

Our formalization also includes in $\mathcal B$ a set of rules that represents the {\em meta-model}, defining $i)$  hierarchical relationships among the BPAL predicates, e.g., $activity(x) \leftarrow task(x)$; $ii)$ disjointness relationships among BPAL elements, e.g., $\bot \leftarrow activity(x) \wedge event(x)$; $iii)$ structural properties which regard  a BPS as a directed 
graph, where edges correspond to sequence and item flow relations.  A first set of structural properties represents constraints that should be verified by a \textit{well-formed} BPS, i.e., syntactically correct BPS: 
(1)  every process is assigned to a unique start event  and to a unique end event;
(2)  every flow element occurs on a path from the start event to the end event;
(3)  start events have no predecessors and end events have no successors;
(4)  branch gateways have exactly one predecessor and at least two successors, while merge gateways have at least two predecessors and exactly one successor;  
(5)  activities and intermediate events have exactly one predecessor and one successor;
(6)  there are no cycles in the hierarchy of compound activities. 

Finally, other meta-model properties are related to the notions 
of path and reachability between flow elements, such as the following ones, which will be 
used in the sequel: 
$\textit{seq}^{+}(E_1,E_2,P)$, representing the transitive closure of the \textit{sequence flow} 
relation, and
$\textit{n\_reachable}(E_1,E_2,E_3,P)$, which holds if there is a path  in  $P$ 
between $E_1$ and $E_2$ not including $E_3$, i.e.:
\smallskip

\noindent $n\_reachable (X,Y,N,P) \leftarrow  seq(X,Y,P) \wedge \naf Y=N $

\noindent $n\_reachable (X,Y,N,P) \leftarrow seq(X,Z,P) \wedge \naf Z=N \wedge n\_reachable (Z,Y,N,P)$

\smallskip
 
With respect to the framework introduced in \cite{bpal-dexa2010,bpal-bsme}, here we consider unstructured cyclic workflows whose behavioral semantics will be introduced in the following.

\subsection{Behavioral Semantics} \label{section:trace-theory}

\noindent
Now we present a formal definition of
the behavioral semantics, or {\em enactment}, of a BPS, by following
an approach inspired by the {\em Fluent Calculus}, 
a well-known calculus for action and change (see~\cite{fluent-calculus-survey}
for an introduction), which is formalized in Logic Programming.

In the Fluent Calculus, the state of the world 
is represented as a collection of \textit{fluents}, i.e., 
terms representing atomic properties that hold at a given instant of time.
An action, also represented as a term, may cause a change of state, 
i.e., an update of the collection of fluents associated with it.
Finally, a {\em plan} is a sequence of actions that
leads from the initial to the final state.
For states we use set notation
(here we depart from~\cite{fluent-calculus-survey}, 
where an associative-commutative operator
is used for representing collections of fluents).
A fluent is an expression 
of the form $f(a_{1},\dots ,a_{n})$, where $f$ is a fluent symbol and 
$a_{1},\dots ,a_{n}$ are constants or variables. 
In order to model the behavior of a BPS, we represent states as
{\em finite sets} of ground fluents. 
We take a closed-world interpretation of states, 
that is, we assume that a fluent $F$, different from {\it true},
holds in a state $S$ iff
$F \in S$.
Our set-based representation of states relies on the assumption
that the BPS is {\em safe}, that is, during its enactment
there are no concurrent executions of the same 
flow element~\cite{vander-workflow}. 
This assumption enforces that the set of  
states reachable by a given BPS is finite.
A {\em fluent expression} is built inductively from fluents, the 
binary function symbol $and$, and the unary function symbol $not$. The  satisfaction 
relation assigns a truth value to a fluent expression with respect to a state.  This
relation is encoded by a predicate $\textit{holds}(F,S)$, which holds if the fluent 
expression $F$ is true in the state $S$.
We also introduce  a constant symbol $\textit{true}$, 
such that $\textit{holds}(\textit{true},S)$ holds for every state $S$. 
Accordingly to the closed-world interpretation 
given to states, the satisfaction relation is defined by the following rules: 

\pagebreak

\smallskip
\noindent $ \textit{holds}(F,S)\leftarrow F = true $

\noindent $ \textit{holds}(F,S)\leftarrow F \in S $

\noindent $ \textit{holds}(not(F),S)\leftarrow \neg \textit{holds}(F,S) $

\noindent $ \textit{holds}(and(F_{1},F_{2}),S)\leftarrow \textit{holds}(F_{1},S) \wedge \textit{holds}(F_{2},S) $  

\smallskip
\noindent
Note that, by the perfect model semantics, reflecting the closed-world assumption, 
for any fluent $F$ different from {\it true},
$not(F)$ holds in a state $S$ iff $F\not\in S$.

We will consider the following two kinds of fluents: 
\begin{itemize}
\item  $\textit{cf}(E_1,E_2,P)$,
which means that the flow element $E_1$ has 
been executed and the successor flow element $E_2$ is waiting 
for execution, during 
the enactment of the process $P$   (\textit{cf} stands for {\em control flow});
\item  $\textit{en}(A,P)$, which means that the activity $A$ is being 
executed during the enactment of the process $P$  (\textit{en} stands for {\em enacting}).
\end{itemize}

\noindent
To clarify our terminology note that, when a flow element $E_2$
is waiting for execution, $E_2$ might not be enabled to execute,
because other conditions need
to be fulfilled, such as those depending on the synchronization with other flow elements
(see, in particular, the semantics of merging behaviors below).

We assume that the execution of an activity has a beginning and a
completion (although we do not associate a {\em duration} with activity execution),
while the other flow elements execute instantaneously. 
Thus, we will consider two kinds of actions: 
$\textit{begin}(A)$ which starts the execution of an activity~$A$, 
and $\textit{complete}(E)$, which represents the 
completion of the execution of a flow element $E$ (possibly, an activity).
The change of state determined by the execution of
an action will be formalized by a
relation $\textit{result}(S_1,A,S_2)$, which holds if  action 
$A$ can be executed in state 
$S_1$ leading to state $S_2$.
For defining the relation $\textit{result}(S_1,A,S_2)$ the following
auxiliary predicates will be used:
(i)~$\textit{update}(S_1,T,U,S_2)$,
which holds if $S_2= (S_1 - T) \cup U$, where
$S_1,T,U,$ and $S_2$ are sets of fluents, and\linebreak (ii)~$\textit{setof}(F,C,S)$,
which holds if $S$ is the set of ground instances of fluent $F$ such that condition $C$ holds.

The relation $r(S_1,S_2)$ holds if
a state $S_2$ is {\em immediately  reachable} from a state $S_1$,
that is, some action $A$ can be executed in state $S_1$ leading to state $S_2$:

\smallskip

\noindent  $r(S_1,S_2) \leftarrow \textit{result}(S_1,A,S_2) $

\smallskip

\noindent We say that a state $S_2$ is {\em reachable} from a state $S_1$  if 
there is a finite, possibly empty, 
sequence of actions from $S_1$ to $S_2$,
that is, $\textit{reachable\_state}(S_1,S_2) $ holds, where the relation $\textit{reachable\_state}$ is is the reflexive-transitive closure of $r$.
 
 In the rest of this section we present a fluent-based formalization of
the behavioral semantics of a BPS as a set of rules $\mathcal T$, partially  reported in Table \ref{tab:sem}.
The proposed formal semantics  is focused on a core of
the BPMN language and it  mainly refers to its semantics, as described (informally) in the most recent specification of 
the language \cite{bpmn}. 
Most of the constructs considered here (e.g., parallel or 
exclusive branching/merging) have the same interpretation in
most workflow languages.
However, when different interpretations are given, e.g., 
in the case of inclusive merge, we stick to the BPMN one.

\subsubsection{Activity and Event Execution}
The enactment of a process $P$ begins  with the execution of the associated start 
event $E$ in a
state where the fluent $\textit{cf}(\textit{start},E,P)$ holds, being \textit{start} 
a reserved constant. After the execution of the start event, its unique successor waits for execution (Rule E1).  The execution of an end event leads to the final state of a process execution, in which the fluent $\textit{cf}(E,\textit{end},P)$ holds, where \textit{E} is the end event associated with the process \textit{P} and \textit{end} is a reserved constant (Rule E2).

According to the informal semantics of 
BPMN, {\em intermediate events} are intended as instantaneous  
patterns of behavior that are registered at a given time point. 
Thus, we formally model the execution of an intermediate 
event as a single state transition, as defined in Rule E3. Intermediate events in BPMN can also be attached to activity boundaries to 
model exceptional flows. Upon occurrence of an {\em exception}, 
the execution of the activity is 
interrupted, and the control flow moves along the sequence flow that leaves the event (Rule E4).

The execution of an activity is enabled to 
begin after the completion of its 
unique predecessor flow element. 
The effects of the execution of an activity vary depending on its type (i.e.,
atomic task or compound activity). 
The beginning of an atomic task $A$ is modeled by adding the $\textit{en}(A,P)$
fluent to the state (Rule A1). At the completion of $A$, the $\textit{en}(A,P)$
fluent is removed and the control flow moves on to the unique successor of $A$ (Rule A2). 
The execution of a compound activity, whose internal structure is defined as a process itself, begins by enabling the execution of the associated \textit{start event} (Rule A3), and 
completes after the execution of the associated \textit{end event} (Rule A4). 

\begin{table}[b!]
\caption{Fragment of the behavioral semantics of the BPAL language }
\label{tab:sem}
\centering
{\footnotesize 
\begin{tabular}{| p{14cm}  | }
\hline
\hangindent=1.1cm \makebox[0.6cm][l]{(E1)} 
$\textit{result}(S_1,\textit{complete}(E),S_2)  \leftarrow
\textit{start\_event}(E) \wedge 
\textit{holds}(\textit{cf}(\textit{start},E,P),S_1) \wedge
\textit{seq}(E,X,P) \wedge 
\textit{update}(S_1, \{\textit{cf}(\textit{start},E,P)\}, \{\textit{cf}(E,X,P)\},S_2) $
\\ 
\hangindent=1.1cm \makebox[0.6cm][l]{(E2)} 
$\textit{result}(S_1,\textit{complete}(E),S_2)  \leftarrow
\textit{end\_event}(E) \wedge 
\textit{holds}(\textit{cf}(X,E,P),S_1) \wedge \newline
 \textit{update}(S_1,\{\textit{cf}(X,E,P)\},\{\textit{cf}(E,\textit{end},P)\},S_2)$
\\ 
\hangindent=1.1cm \makebox[0.6cm][l]{(E3)} 
$\textit{result}(S_1,\textit{complete}(E),S_2)  \leftarrow
\textit{int\_event}(E) \wedge
\textit{holds}(\textit{cf}(X,E,P), S_1) \wedge
\textit{seq}(E,Y,P) \wedge \newline
\textit{update}(S_1, \{\textit{cf}(X,E,P)\}, \{\textit{cf}(E,Y,P)\},S_2) $
\\
\hangindent=1.1cm \makebox[0.6cm][l]{(E4)}
$\textit{result}(S_1,\textit{complete}(E),S_2)  \leftarrow
\textit{exception}(E,A,P) \wedge
\textit{int\_event}(E) \wedge \textit{holds}(\textit{en}(A,P), S_1) \wedge\newline
\textit{seq}(E,Y,P) \wedge
\textit{update}(S_1, \{\textit{en}(A,P)\}, \{\textit{cf}(E,Y,P)\},S_2) $
\\ \hline \hline

\hangindent=1.1cm \makebox[0.6cm][l]{(A1)} 
$\textit{result}(S_1,begin(A),S_2)  \leftarrow
\textit{task}(A) \wedge
\textit{holds}(\textit{cf}(X,A,P), S_1) \wedge \newline
\textit{update}(S_1, \{\textit{cf}(X,A,P)\}, \{\textit{en}(A,P)\}, S_2)$
\\ 
\hangindent=1.1cm \makebox[0.6cm][l]{(A2)} 
$\textit{result}(S_1,\textit{complete}(A),S_2)  
\leftarrow
\textit{task}(A) \wedge
\textit{holds}(\textit{en}(A,P),S_1) \wedge
\textit{seq}(A,Y,P) \wedge\newline
\textit{update}(S_1,\{\textit{en}(A,P)\},\{\textit{cf}(A,Y,P)\},S_2)$
\\ 
\hangindent=1.1cm \makebox[0.6cm][l]{(A3)} 
$\textit{result}(S_1,begin(A),S_2)  \leftarrow
\textit{comp\_act}(A,S,E) \wedge \textit{holds}(and(\textit{cf}(X,A,P), not(\textit{en}(A,P))),S_1) \wedge \newline  
\textit{update}(S_1,\{\textit{cf}(X,A,P)\},\{\textit{cf}(start,S,A), \textit{en}(A,P)\},S_2)$
\\ 
\hangindent=1.1cm \makebox[0.6cm][l]{(A4)} 
$\textit{result}(S_1,\textit{complete}(A),S_2) \leftarrow \textit{comp\_act}(A,S,E) \wedge  \textit{holds}(and(\textit{cf}(E,\textit{end},A), \textit{en}(A,P)),S_1) \wedge 
\textit{seq}(A,Y,P) \wedge  \textit{update}(S_1,\{\textit{en}(A,P), \textit{cf}(E,\textit{end},A)\},  \{\textit{cf}(A,Y,P)\},S_2) $ 
\\  \hline \hline

\noindent \hangindent=1.1cm \makebox[0.6cm][l]{(B1)} 
$\textit{result}(S_1,\textit{complete}(B),S_2)  \leftarrow
\textit{exc\_branch}(B) \wedge 
\textit{holds}(\textit{cf}(X,B,P), S_1) \wedge 
 \textit{seq}(B,Y,P) \wedge \newline$
$update(S_1,\{\textit{cf}(X,B,P)\},\{\textit{cf}(B,Y,P)\},S_2)$
\\ 

\hangindent=1.1cm \makebox[0.6cm][l]{(B2)} 
$\textit{result}(S_1,\textit{complete}(B),S_2)  \leftarrow
\textit{inc\_branch}(B) \wedge 
\textit{holds}(\textit{cf}(X,B,P),S_1) \wedge \newline
\textit{setof}(\textit{cf}(B,Y,P), \textit{seq}(B,Y,P) , \textit{Succ}) \wedge \textit{subseteq}(SubSucc,Succ) \wedge
\neg \textit{emptyset}(SubSucc)  \wedge  \newline
\textit{update}(I,\{\textit{cf}(X,B,P)\},\textit{SubSucc},S_2)$
\\ 
\noindent \hangindent=1.1cm \makebox[0.6cm][l]{(B3)} 
$\textit{result}(S_1,\textit{complete}(B),S_2)  \leftarrow
\textit{par\_branch}(B) \wedge 
\textit{holds}(\textit{cf}(X,B,P), S_1) \wedge \newline
\textit{setof}(\textit{cf}(B,Y,P), \textit{seq}(B,Y,P), \textit{Succ}) \wedge$
$update(S_1,\{\textit{cf}(X,B,P)\},\textit{Succ},S_2)$
\\ \hline \hline

\hangindent=1.1cm \makebox[0.6cm][l]{(X1)} 
$\textit{result}(S_1,\textit{complete}(M),S_2)  
\leftarrow
\textit{exc\_merge}(M) \wedge
\textit{holds}(\textit{cf}(A,M,P),S_1) \wedge
\textit{seq}(M,Y,P) \wedge\newline
\textit{update}(S_1,\{\textit{cf}(A,M,P)\},\{\textit{cf}(M,Y,P)\},S_2)$
\\ 

\noindent \hangindent=1.1cm \makebox[0.6cm][l]{(O1)} 
$\textit{result}(S_1,\textit{complete}(M),S_2) \leftarrow
\textit{inc\_merge}(M) \wedge
\textit{enabled\_im}(M,S_1,P)\wedge 
\textit{seq}(M,Y,P)  \wedge \newline
\textit{setof}(\textit{cf}(X,M,P), \textit{holds}(\textit{cf}(X,M,P),S_1),\textit{PredM}) \wedge 
\textit{update}(S_1,\textit{PredM},\{\textit{cf}(M,Y,P)\},S_2)$
\\ 
\noindent \hangindent=1.1cm \makebox[0.6cm][l]{(O2)} 
$\textit{enabled\_im}(M,S_1,P) \leftarrow
\textit{holds}(\textit{cf}(X,M,P),S_1)  \wedge
\neg \textit{exists\_upstream}(M,S_1,P)$
\\ 
\noindent \hangindent=1.1cm \makebox[0.6cm][l]{(O3)} 
$\textit{exists\_upstream}(M,S_1,P) \leftarrow
\textit{seq}(X,M,P) \wedge
\textit{holds}(\textit{not}(\textit{cf}(X,M,P)),S_1) \wedge \newline
\textit{holds}(\textit{cf}(Y,U,P),S_1) \wedge
\textit{upstream}(U,X,M,S_1,P)$
\\
\noindent \hangindent=1.1cm \makebox[0.6cm][l]{(O4)} 
$\textit{upstream}(U,X,M,S_1,P) \leftarrow
\textit{n\_reachable}(U,X,M,P) \wedge
\neg \textit{exists\_path}(U,M,S_1,P)$
\\ 
\noindent \hangindent=1.1cm \makebox[0.6cm][l]{(O5)} 
$\textit{exists\_path}(U,M,S_1,P) \leftarrow
\textit{holds}(\textit{cf}(K,M,P),S_1) \wedge
\textit{n\_reachable}(U,K,M,P)$
\\ 
\noindent \hangindent=1.1cm \makebox[0.6cm][l]{(P1)} 
$ \textit{result}(S_1,\textit{complete}(M),S_2)  
\textit{par\_merge}(M) \wedge
\neg \textit{exists\_non\_executed\_pred}(M,P,S_1)\! \wedge\!
\newline
\textit{seq}(M,Y,P)\! \wedge
\textit{setof}(\textit{cf}(X,M,P), \textit{seq}(X,M,P),\textit{PredM}) \wedge 
\textit{update}(S_1,\textit{PredM},\{\textit{cf}(M,Y,P)\},S_2) $
\\
\noindent \hangindent=1.1cm \makebox[0.6cm][l]{(P2)} 
$\textit{exists\_non\_executed\_pred}(M,P,S_1)\leftarrow
\textit{seq}(X,M,P) \wedge
\textit{holds}(not ( \textit{cf}(X,M,P)),S_1)$
\\ \hline
\end{tabular}
}
\end{table}

\subsubsection{Branching Behaviors}
\label{subsec:cond}

When a branch gateway is executed, a subset of its successors is selected for 
execution. We consider here exclusive, inclusive, and parallel branch gateways.

An exclusive branch leads to the execution of exactly one successor (Rule B1), while an inclusive branch leads to the concurrent execution of a non-empty subset
of its successors (Rule B2). The set of successors of exclusive or inclusive decision points may depend  on {\em guards}, i.e., conditions that usually take the form of tests on the value of the items that are passed between the activities. 
While Rules B1-B2 formalize a nondeterministic choice among the successors of a decision point, in Section  \ref{section:state-updates} guard expressions will be included in the framework in the form of fluent expressions whose truth value is tested with respect to the current state. Finally, a parallel branch leads to the concurrent execution of all its successors (Rule B3).

\subsubsection{Merging Behaviors}

An exclusive merge
can be executed whenever at least one of its predecessors has been 
executed (Rule X1). 

For the inclusive merge several operational semantics have been proposed, due to the complexity of its non-local semantics (see e.g., \cite{vicious-circle,or-ibm}). An inclusive 
merge is supposed to be able to synchronize a varying number of threads, i.e., it is 
executed only when $n (\geq 1)$ predecessors have been executed and no other will 
be eventually executed. Here we refer to the semantics described in \cite{or-ibm} adopted 
by BPMN, stating that (Rule O1) an inclusive merge \textit{M} can be executed if the following two 
conditions hold (Rules O2, O3):

\vspace*{-2mm}
\begin{enumerate}
\item[(1)]  at least one of its predecessors has been executed,
\item[(2)]  for each non-executed predecessor \textit{X}, there is no flow element \textit{U} which is waiting for execution and is {\em upstream} \textit{X}.
The notion of being upstream captures the
fact that $U$ may lead to the execution of $X$, and is
defined as follows.
A flow element \textit{U} is upstream  $X$ if (Rules O4, O5):
\textit{a)} there is a path from \textit{U} to \textit{X} not including \textit{M}, and
\textit{b)} there is no path from \textit{U} to an executed predecessor of \textit{M} not including~\textit{M}. 
\end{enumerate}
\vspace*{-2mm}

%
%
%
%
%
%
%
%
%
%
%
%
\noindent
Finally, a parallel merge can be executed if all its predecessors have been 
executed as defined in Rule P1, where $\textit{exists\_non\_executed\_pred}(M,P,S_1)$ holds if there exists no predecessor of $M$ which has not been executed in state $S_1$ (Rule P2).

\subsubsection{Item Flow} \label{subsec:ItemFlow}
BP modeling must be able to represent the physical and the information 
items that are produced and consumed by the various activities during the execution of a 
process. For the formalization of the \textit{item flow} semantics, we commit to the BPMN 
standard, where the so-called \textit{data objects} are used to store information created 
and read by activities. 

In our approach 
items are essentially regarded as variables, and hence there is a single instance of a 
given item any time during the execution that may be (over-)written by some activity. We 
consider two main types of relations between activities and items. First of all, an 
activity may \textit{use} a particular item (\textit{input} relation). 
This implies that the item is 
expected to hold a value before the activity is executed. Second, an activity may produce a 
particular value (\textit{output} relation), causing the item to get a new value. 
If it has no 
value yet, it is created, otherwise it is overwritten. It is worth noting that the item flow is 
not necessarily imposed over the control flow, but they interact for the definition of 
the process behavior. For instance, an activity expecting a value from a given item, may 
also cause a deadlock if this condition is never satisfied. 

The item flow is modeled through the fluent $\textit{wrtn}(A,\textit{It},P)$
(\textit{wrtn} stands for ``written") representing the situation 
in which the item \textit{It} has been produced by the activity
\textit{A} in the 
enactment of the process \textit{P}. 
In order to handle item manipulation, the semantics of task enactment (Rules A1, A2) is extended as follows:

\smallskip

\noindent \hangindent=0.3cm $\textit{result}(S_1,\textit{begin}(A),S_2)  \leftarrow
\textit{task}(A) \wedge \textit{holds}(\textit{cf}(X,A,P),S_1)\wedge \neg \textit{blocked\_input}(A,P,S_1) \wedge
\newline
\textit{input}(A,\textit{It},P) \wedge
\textit{update}(S_1,\{\textit{cf}(X,A,P)\},\{\textit{en}(A,P)\},S_2) $

\smallskip

\noindent \hangindent=0.3cm $\textit{result}(S_1,\textit{complete}(A),S_2)  \leftarrow \textit{task}(A) \wedge
\textit{holds}(\textit{en}(A,P),S_1) \wedge
\textit{seq}(A,Y,P) \wedge
\newline
\textit{setof} (\textit{wrtn}(A,\textit{It},P), 
\textit{output}(A,\textit{It},P), U)\wedge
\textit{update}(S_1,\{\textit{en}(A,P)\},\{\textit{cf}(A,Y,P)\}\cup U,S_2) $

\smallskip

\noindent
where $\textit{blocked\_input}(A,P,S_1)$ holds if, at a state $S_1$ 
during the enactment of process $P$, there exists some input
item \textit{It} for $A$ that has not been produced. Thus,

\smallskip

\noindent \hangindent=0.3cm \( \textit{blocked\_input}(A,P,S_1)\leftarrow \textit{input}(A,\textit{It},P) \wedge \neg \textit{updated\_item}(\textit{It},P,S_1) \) 

\noindent \hangindent=0.3cm \(\textit{updated\_item}(\textit{It},P,S_1) \leftarrow \textit{holds}(\textit{wrtn}(B,\textit{It},P),S_1))\)

\smallskip

\noindent
The case of compound activities can be treated in a similar way and is omitted.

\section{Semantic Annotation}
\label{section:sem-annotation}

\noindent
In the previous section we have shown how a procedural representation of a BPS can be modeled in our rule-based framework as an activity workflow.
However, not all the relevant knowledge regarding process enactment is captured by 
a workflow model, which defines the planned order of operations
but does not provide an explicit representation of the domain knowledge regarding the
entities involved in such a process, i.e., the business environment in which processes 
are carried out.

Similarly to proposals like 
{\em Semantic BPM} \cite{semantic-bpm} and {\em Semantic
Web Services} \cite{wsmo-book}, we will make use of 
{\em semantic annotations} to enrich
the procedural knowledge specified by a BPS with domain knowledge
expressed in terms of a given business reference ontology.
Annotations provide two kinds of ontology-based information: 
\textit{(i)}~formal definitions of the basic entities involved in a process 
(e.g., activities, actors,
items) to specify their meaning in an unambiguous way ({\em terminological} annotations), and 
\textit{(ii)}~specifications of preconditions and effects of the enactment of flow elements
 ({\em functional} annotations). 

\subsection{Reference Ontology}
\label{sec:enrichment}

A business reference ontology is intended to capture  the semantics of a business scenario 
in terms of the relevant  vocabulary plus a set of axioms (TBox) that define 
the intended meaning of the vocabulary terms. 
In order to represent the semantic annotations of a BPAL BPS in a uniform way, we will consider ontologies falling within the OWL 2 RL  profile (See Section \ref{sect:owl}), and hence expressible as sets of rules.
An OWL 2 RL ontology is represented as a set $\O$ of rules, consisting of a set of 
facts of the form $t(s,p,o)$, called {\em triples}, encoding the specific 
OWL TBox, along with the rules that are common to all OWL 2 RL ontologies, such as the ones of Table \ref{tab:owl-rl}. 

In Table \ref{tab:reference-onto} we show an example of business reference ontology for the annotation of the Handle Order process depicted in Figure \ref{fig:example}. For the sake of conciseness and clarity, 
the axioms of ontology are represented as DL expressions, 
instead of sets of triples. The translation into triple form can be
done automatically as shown in \cite{dlp,pd*}.


\begin{table}[ht!]
\caption{Business Reference Ontology excerpt}
\label{tab:reference-onto}
\centering
{\footnotesize

\begin{tabular}{|c|c|} \hline 
\multicolumn{2}{|c|}{\textbf{Actors}}  \\ \hline 
\textit{Organizational\_Actor} $\sqsubseteq$ \textit{Actor} & 	\textit{Human\_Actor} $\sqsubseteq$ \textit{Actor}   \\ \hline  
\textit{Corporate\_Customer} $\sqsubseteq$ \textit{Organizational\_Actor} 			 &   \textit{Employee} $\sqsubseteq$ \textit{Human\_Actor}    \\ \hline
\textit{Department} $\sqsubseteq$ \textit{Organizational\_Actor} &  \textit{Business\_Partner} $\sqsubseteq$ \textit{Organizational\_Actor}\\ \hline
    \textit{Accounting\_Dpt} $\sqsubseteq$ \textit{Department}&  \textit{Supply\_Dpt} $\sqsubseteq$ \textit{Department}  \\ \hline  
\textit{Order\_Mgt\_Dept} $\sqsubseteq$ \textit{Department} &  \textit{Warehouse\_Mgt} $\sqsubseteq$ \textit{Department} \\ \hline
\textit{Carrier} $\sqsubseteq$ \textit{Organizational\_Actor} & \textit{Courier} $\sqsubseteq$ \textit{Carrier} $\AND$ \textit{Business\_Partner}\\
\hline
\textit{Carrier\_Dpt} $\sqsubseteq$ \textit{Carrier} $\AND$ \textit{Department}  &  \\ \hline
  
\multicolumn{2}{|c|}{\textbf{Objects}}  \\ \hline 
\textit{ClosedPO} $\ISA$ \textit{Purchase\_Order} &   \textit{ApprovedPO} $\ISA$ \textit{Purchase\_Order}  \\ \hline
 
\textit{CancelledPO} $\ISA$ \textit{ClosedPO}   & \textit{FulfilledPO} $\ISA $ \textit{ClosedPO} \\ \hline

\textit{UnavailablePL} $\ISA$ \textit{Part\_List} & \textit{AvailablePL} $\ISA$ \textit{Part\_List}  \\ \hline 
 \textit{payment} $\ISA$ \textit{related} &  $\exists$ \textit{payment}$^{-} \ISA$ \textit{Invoice}    \\ \hline  
 
 \textit{CancelledPO} $\AND$ \textit{ApprovedPO}  $\ISA \bot$ & \textit{UnavailablePL} $\AND$  \textit{AvailablePL} $\ISA \bot$ \\ \hline 
\textit{ApprovedPO} $\AND\ \SOME{\textit{related}}{\textit{Invoice}}$   $\ISA$ \textit{FulfilledPO} & 
   \textit{Order} $\AND\ \SOME{\textit{related}}{\textit{UnavailablePL}}$   $\ISA$ \textit{CancelledPO}
   \\ \hline
   \multicolumn{2}{|c|}{\textbf{Processes}}  \\ \hline 
\textit{AuthorizingProcedure} $\sqsubseteq$ \textit{Process} & \textit{Transportation} $\sqsubseteq$ \textit{Process} \\ \hline
\textit{Payment} $\sqsubseteq$ \textit{Process} & \textit{Invoicing} $\sqsubseteq$ \textit{Process} \\ \hline
\textit{Communication} $\sqsubseteq$ \textit{Process} & \textit{Refuse} $\sqsubseteq$ \textit{Communication} \\ \hline
\textit{Rejecting} $\sqsubseteq$ \textit{Authorizing\_Procedure}  & \textit{Accepting} $\sqsubseteq$ \textit{Authorizing\_Procedure}\\ \hline
 \multicolumn{2}{|c|}{\textbf{Relations}}  \\ \hline 
\textit{member} $\sqsubseteq$ \textit{related} & \textit{content} $\sqsubseteq$ \textit{related}\\ \hline
\textit{destination} $\sqsubseteq$ \textit{related} & $\exists$\textit{member} $\ISA$ \textit{Human\_Actor}\\ \hline
 $\exists$\textit{member}$^{-} \ISA$ \textit{Organizational\_Actor}   &  \\ \hline

\end{tabular}
}
\end{table}

\subsection{Terminological Annotation}
\label{sec:TermAnn}
A \textit{terminological  annotation}  associates elements of a BPS with concepts of a reference ontology, in order to describe the former in terms of a suitable conceptualization of the underlying business  domain provided by the latter. This association is specified by a set of OWL assertions of the form $ BpsEl : \SOME{\textit{termRef}}{\textit{Concept}}$, where:

\vspace*{-2mm}

\begin{itemize}

\item \textit{BpsEl} is an element of a BPS;

\vspace*{-2mm}

\item \textit{Concept} is either \textit{i)} a named concept defined in  the ontology, e.g., \textit{Purchase\_Order}, or \textit{ii)} a complex concept, defined by a class expression, e.g., \textit{Rejecting} $\sqcap \ \linebreak \exists$\textit{content.Purchase\_Order};

\vspace*{-2mm}

\item $\textit{termRef}$ is an OWL property name. 
\end{itemize}

\vspace*{-2mm}

\noindent
We do not assume that every BPS element is annotated,
nor that every concept  is the meaning associated with some BPS element. 
Furthermore, different BPS elements could be annotated with respect to the same concept, to provide an alignment of the different terminologies and conceptualizations used in different BP schemas. E.g., the activities \emph{bill\_client} and \emph{issue\_invoice} occurring in different processes may  actually refer to the same notion, suitably defined in the ontology. 

\begin{example1}{\rm 
Examples of annotations  related to the Handle Order process (Figure \ref{fig:example}) are listed below. The item \emph{order} is annotated with the \emph{Purchase\_Order}  concept, while the participant \emph{shipper} with the concept \emph{Carrier}, which can be either an internal \emph{Department} or a \emph{Business\_Partner}. A \emph{sales\_clerk}  is defined as an \emph{Employee}, which is part of the \mbox{\emph{Order\_Mgt\_Dpt}}. The task \emph{delivering} is defined as a \emph{Transportation} related to some sort of \emph{Product}. Finally, \emph{notify\_rejection} represents a \emph{Communication} with a \emph{Corporate\_Customer}, and in particular, a \emph{Refuse} related to \emph{Purchase\_Order}. 

\smallskip
 $ order : \SOME{\textit{termRef}}{Purchase\_Order}$ \\
\indent  $ shipper : \SOME{\textit{termRef}}{Carrier}$\\
\indent  $sales\_clerk : \SOME{\textit{termRef}}{(\textit{Employee} \sqcap \SOME {\textit{member}}{\textit{Order\_Mgt\_Dpt}})}$\\
 \indent $delivering: \SOME{\textit{termRef}}{(\textit{Transportation} \sqcap \SOME {\textit{related}}{\textit{Product}})}$ \\
 \indent $\textit{notify\_rejection}: \SOME{\textit{termRef}}{(\textit{Refuse} \sqcap \SOME {\textit{content}}{\textit{Purchase\_Order}} \ \sqcap $ \\  
 			\indent \indent \indent $ \SOME {\textit{destination}}{\textit{Corporate\_Customer}}  )}$ 
}
\end{example1}

\subsection{Functional Annotation}
\label{section:state-updates}

By using the ontology vocabulary and axioms, we define semantic annotations for 
modeling the behavior of individual process elements in terms of \textit{preconditions} 
under which a flow element can be executed, and \textit{effects} on the state of the world 
after its execution.  Preconditions and effects, collectively called {\em functional annotations}, can be used, for instance, to model 
input/output relations of activities with business entities.   Fluents can represent the 
\textit{properties} of a business entity affected by the execution of an activity
at a given time during the execution of the process. 
A precondition specifies the properties a business entity must posses 
when an activity is enabled to start, and
an effect specifies the properties of a business entity after having completed an activity.  These aspects are only partially supported by current BP modeling notations, such as BPMN,  in terms of data objects representing information storage during the BP enactment.

Functional annotations are formulated by means of the following  relations:

\begin{itemize}
\item $\textit{pre}(A,C,P)$, which specifies a fluent expression $C$, 
called \textit{enabling condition}, 
that must hold to execute an element \textit{A} in the 
process \textit{P};
\item $\textit{eff}(A,Q,E^{-},E^{+},P)$, which specifies the set 
$E^{-}$ of fluents, called \textit{negative effects},
that do not hold after the execution of $A$  and 
the set of fluents $E^{+}$,  called \textit{positive effects}, 
that hold after the execution of $A$ in
the process $P$. $Q$ is a fluent expression that must hold to complete the activity $A$.
We assume that $E^{-}$ and $E^{+}$ are disjoint sets of fluents, 
and the variables occurring in them also occur in $Q$.
\item $\textit{c\_seq}(G,B,Y,P)$, which models a \textit{conditional sequence flow} used to select the set of successors of decision points.
 $G$ is a $guard$ associated to the exclusive or inclusive branch gateway $B$, i.e., a
fluent expressions that must hold in order to enable the flow element $Y$,
successor of $B$ in the process $P$. 
We also have the rule $\textit{seq}(B,Y,P) \leftarrow \textit{c\_seq}(G,B,Y,P)$. 
\end{itemize}

The enabling conditions, the guards and the negative and positive effects 
occurring in functional annotations
are fluent expressions built from 
fluents of the form $\tf(s,p,o)$, corresponding to the OWL statement $t(s,p,o)$,
where we adopt the usual \textit{rdf}, \textit{rdfs}, and \textit{owl}
prefixes for names in the OWL vocabulary, and the \textit{bro} prefix 
for names relative to our specific examples.
We assume that the fluents appearing in  functional annotations
are either of the form $\tf(a,\textit{rdf:type},c)$, 
corresponding to the unary atom $c(a)$,
or of the form $\tf(a,p,b)$, corresponding to
the binary atom $p(a,b)$, where $a$ and $c$ are {\em individuals},
while $c$ and $p$ are concepts and properties, respectively, defined in the reference ontology~$\O$.
Thus, fluents correspond to assertions about individuals, i.e., 
assertions belonging to the ABox of 
the ontology, and hence the ABox
may change during process enactment
due to the effects specified by the functional annotations,
while $\O$, providing the ontology definitions and axioms, i.e., the TBox
of the ontology, does not change.

Let us now present an example of specification of functional annotations.
In particular, our example shows nondeterministic effects,
that is, a case where a flow element $A$ is associated with more than one 
pair $(E^{-},E^{+})$ of negative and positive effects.

\begin{example1} \label{ex:close-order}{\rm 
Consider again the \emph{Handle Order} process shown in Figure \ref{fig:example}. After the 
execution of \emph{create\_order}, a purchase order is issued. This order can be 
approved or 
canceled upon execution of the activities \emph{accept\_order} and 
\emph{cancel\_order},
respectively.  Depending on the inventory capacity checked during the
 \emph{check\_inventory} task, the requisition of parts performed by an external supplier is performed 
(\emph{parts\_auction}). Once that all the order parts are available, the order can be 
fulfilled and an invoice is associated with the order. 
This behavior is specified by the functional 
annotations reported in Table \ref{tab:functional-annotations}.}
\end{example1}
 
 \begin{center}
\begin{table*}[h!]
\centering
\footnotesize{
\caption{Functional annotation for the Handle Order process}
\label{tab:functional-annotations}  
\begin{tabular}{|c|l|l|} \hline 
\textbf{Flow Element} & \textbf{Enabling Condition} (pre)  & \textbf{Effects} (eff) \\ \hline 
create\_order &\textit{true} & Q: \textit{ true}  \\
		   & &  E$^{+}$: \{$\tf(o,\textit{rdf:type},\textit{bro:Purchase\_Order})$\} \\ \hline

accept\_order & $\tf(O,\textit{rdf:type},\textit{bro:Purchase\_Order})$ & Q: $\tf(O,\textit{rdf:type},\textit{bro:Purchase\_Order})$ \\
		   & &  E$^{+}$: \{$\tf(O,\textit{rdf:type},\textit{bro:ApprovedPO})$\} \\ \hline

cancel\_order &$ \tf(O,\textit{rdf:type},\textit{bro:ApprovedPO}) $& Q: $\tf(O,\textit{rdf:type},\textit{bro:ApprovedPO})  $  \\
		   & &  E$^{-}$: \{$\tf(o,\textit{rdf:type},\textit{bro:ApprovedPO})$\}  \\ 		   & &  E$^{+}$: \{$\tf(o,\textit{rdf:type},\textit{bro:CancelledPO})$\}	  \\ 
\hline

check\_inventory &$ \tf(O,\textit{rdf:type},\textit{bro:ApprovedPO}) $& Q: $\tf(O,\textit{rdf:type},\textit{bro:ApprovedPO})  $  \\
		   & &  E$^{+}$: \{$\tf(O,\textit{bro:related},\textit{pl}), $	  \\ 
		   & & \ \ \ \ \ \ \ \ $\tf(pl,\textit{rdf:type},\textit{bro:Part\_List}) $\}  \\ \hline

check\_inventory &  $\tf(O,\textit{rdf:type},\textit{bro:ApprovedPO})$ & \\ \hline

parts\_auction & $\tf(PL,\textit{rdf:type},\textit{bro:Part\_List})$ & Q: $\tf(PL,\textit{rdf:type},\textit{bro:Part\_List}) $  \\
		   & &  E$^{+}$: \{$\tf(PL,\textit{rdf:type},\textit{bro:AvailablePL})$\}  \\ \hline

parts\_auction & $\tf(PL,\textit{rdf:type},\textit{bro:Part\_List})$ & Q: $and(\tf(O,\textit{rdf:type},\textit{bro:ApprovedPO}),$ \\ & & \ \ \ \ \ \ \ \ $\tf(PL,\textit{rdf:type},\textit{bro:Part\_List})) $  \\
		   & &  E$^{-}$: \{$ \tf(O,\textit{rdf:type},\textit{bro:ApprovedPO})$\}  \\ 
		   & &  E$^{+}$: \{$\tf(PL,\textit{rdf:type},\textit{bro:UnavailablePL}) $\}  \\

\hline

bill\_client & $  \tf(O,\textit{rdf:type},\textit{bro:ApprovedPO}) $ &  
Q: \textit{ true}  \\
  & &  E$^{+}$: \{$ \tf(i,\textit{rdf:type},\textit{bro:Invoice})$\}     \\ \hline

payment & $and(\tf(O,\textit{rdf:type},\textit{bro:ApprovedPO}),$  & Q: $and(\tf(O,\textit{rdf:type},\textit{bro:ApprovedPO}),$ \\ 
&\ \ \ \ \ \ \ \  $\tf(I,\textit{rdf:type},\textit{bro:Invoice})) $  &\ \ \ \ \ \ \ \  $\tf(I,\textit{rdf:type},\textit{bro:Invoice}))$   \\
		   & &  E$^{+}$: \{$\tf(O,\textit{bro:payment},\textit{I})) $\}	  \\  \hline
\end{tabular}
\begin{tabular}{|c|c|c|} \hline 
\textbf{Branch} & \textbf{Successor}   & \textbf{Guard}    \\ \hline 
g1 & g3 & $\tf(O,\textit{rdf:type},\textit{bro:ApprovedPO})$  \\ \hline
g1 & g2 & $\textit{not}( \tf(O,\textit{rdf:type},\textit{bro:ApprovedPO}))$  \\ \hline
g3 & parts\_auction & $(\tf(PL,\textit{rdf:type},\textit{bro:Part\_List})$ \\ \hline
g5 & g2 &  $\tf(O,\textit{rdf:type},\textit{bro:CancelledPO})$ \\ \hline
g5 & select\_shipper & $ \tf(O,\textit{rdf:type},\textit{bro:ApprovedPO})$ \\ \hline
\end{tabular}
}
\end{table*}
\end{center}

\subsubsection{Formal Semantics of Functional Annotations}
In the presence of functional annotations, the enactment of a BPS is
modeled by extending the $\textit{result}$ relation so as to
take into account the \textit{pre} and \textit{eff} relations. 
We only consider the case of task execution. 
The other cases are similar and will be omitted.

Given a state $S_1$,
a flow element $A$ can be enacted if $A$ is waiting for execution
according to the control flow semantics, and its enabling condition $C$ is satisfied, i.e., $\textit{holds}(C,S_1)$ is true. 
Moreover, given an annotation $\textit{eff}(A,Q,E^{-},E^{+},P)$, 
when $A$ is completed in a given state $S_1$, then a new state
$S_2$ is obtained by taking out from $S_1$ the set $E^{-}$ of
fluents and then adding the set $E^{+}$ of fluents. 
The execution of tasks considering functional annotations is then defined as:

\smallskip
\noindent \hangindent=0.3cm $\textit{result}(S_1,\textit{begin}(A),S_2)  \leftarrow \textit{task}(A) \wedge$
$\textit{holds}(\textit{cf}(B,A,P), S_1) \wedge
\textit{pre}(A,C,P) \wedge \textit{holds}(C,S_1) \wedge$
$\textit{update}(S_1, \{\textit {cf}(B,A,P)\}, \{\textit{en}(A,P)\}, S_2)$

\smallskip
\noindent \hangindent=0.3cm $\textit{result}(S_1,\textit{complete}(A),S_2)  \leftarrow \textit{task}(A) \wedge$
$\textit{holds}(\textit{en}(A,P), S_1) \wedge \textit{eff}(A,Q,E^{-},E^{+},P) \wedge$\\
$holds(Q,S_1) \wedge \textit{seq}(A,B,P) \wedge \textit{update}(S_1,\{\textit{en}(A,P)\}\cup E^{-},$
$\{\textit{cf}(A,B,P)\} \cup E^{+},S_2)$

\smallskip

\noindent
Note that, since the variables occurring in $E^{+}$ and $E^{-}$ are included in those of $Q$ , the evaluation of $holds(Q,S_1)$ binds these variables to constants.

Similarly, the semantics of inclusive and exclusive branches is extended to evaluate the associated guard expressions, in order to determine the set of successors to be enabled. The execution of decision points is then defined as:

\smallskip
\noindent \hangindent=0.3cm $\textit{result}(S_1,\textit{complete}(B),S_2)  \leftarrow
\textit{inc\_branch}(B) \wedge 
\textit{holds}(\textit{cf}(A,B,P),S_1) \wedge$ \makebox[8mm][l]{$\textit{setof}($}$\textit{cf}(B,C,P),$\\  
\hspace*{3mm}$(\textit{c\_seq}(G,B,C,P) \wedge \textit{holds}(G,S_1)), \textit{Succ}) \wedge$
$\textit{update}(I,\{\textit{cf}(A,B,P)\},\textit{Succ},S_2)$

\smallskip

\noindent \hangindent=0.3cm 
$\textit{result}(S_1,\textit{complete}(B),S_2)  \leftarrow
\textit{exc\_branch}(B) \wedge 
\textit{holds}(\textit{cf}(A,B,P), S_1) \wedge 
 \textit{c\_seq}(G,B,C,P) \wedge$ \\
$\textit{holds}(G,S_1) \wedge update(S_1,\{\textit{cf}(A,B,P)\},\{\textit{cf}(B,C,P)\},S_2)$

 \smallskip 

In order to evaluate a statement of the form
$\textit{holds}(\tf(s,p,o),X)$, where $\tf(s,p,o)$ is a fluent
and $X$ is a state, the definition of the $\textit{holds}$ predicate
given previously must be extended to take into account the axioms 
belonging to the reference ontology $\O$. 
Indeed, we want that a fluent of the form $\tf(s,p,o)$ be true in state $X$
not only if it belongs to $X$, but
also if it can be inferred from the fluents in $X$ and the axioms of the ontology. 

For instance, let us consider the fluent  
$f=\tf(o,\textit{rdf:type},\textit{bro:CancelledPO})$. 
We can easily infer that $f$ holds in a state that contains 
$\{\tf(o,\textit{rdf:type},\textit{bro:CancelledPO})\}$ (e.g., reachable after the execution of \textit{cancel\_order})  by using the rule $\textit{holds}(F,X) \leftarrow F\in X$.
However, by
taking into account the ontology excerpt given in Table \ref{tab:reference-onto},
we also want to be able to infer that $f$ holds in a state that contains
 $\{ \tf(o,\textit{rdf:type},\textit{bro:Purchase\_Order}),$
$\tf(o,\textit{bro:related},pl),$ $\tf(pl,\textit{rdf:type},\textit{bro:UnavailablePL})\}$  (e.g., a state reachable after the execution of \textit{parts\_auction}). 

In our framework the inference of new fluents from fluents belonging
to states is performed by including extra rules 
derived by translating the OWL 2 RL/RDF entailment rules as follows: 
every triple of the form
$t(s,p,o)$, where $s$ refers to an individual, is replaced by the atom
$\textit{holds}(\tf(s,p,o),X)$.
Below we show exemplary rules (in particular, those required by our running example)  for concept subsumption (1), role subsumption (2), 
domain restriction (3), transitive property (4), concept intersection\footnote{Without loss of generality, unlike \cite{owl}, we encode binary intersection  instead of a general \textit{n}-ary operator.} (5),  existentially quantified formulae (6), and concept disjointness (7) . We refer the reader to \cite{owl} for the complete list of rules and the discussion of  the OWL 2 RL rule-based semantics.       
 
 \smallskip

\pagebreak

\noindent \hangindent=0.6cm $1. \ \textit{holds}(\tf(S,\textit{rdf:type},C),X) \leftarrow 
\textit{holds}(\tf(S,\textit{rdf:type},B),X) \wedge   t(B,\textit{rdfs:subClassOf},C) $

\noindent \hangindent=0.6cm $2. \ \textit{holds}(\tf(S,P,O),X) \leftarrow \textit{holds}(\tf(S,P1,O),X) \wedge 
  t(P1,\textit{rdfs:subPropertyOf},P)$

\noindent \hangindent=0.6cm $3. \ \textit{holds}(\tf(S,\textit{rdf:type},C),X) \leftarrow  
 \textit{holds}(\tf(S,P,O),X) \wedge  t(P,\textit{rdfs:domain},C) $

\noindent \hangindent=0.6cm $4. \ \textit{holds}(\tf(S,P,O),X)\leftarrow 
  \textit{holds}(\tf(S,P,O_1),X) \wedge  \textit{holds}(\tf(O_1,P,O),X) \wedge   \\
 t(P,\textit{rdf:type},\textit{owl:TransitiveProperty}) $ 

\noindent \hangindent=0.6cm $5. \ \textit{holds}(\tf(S,\textit{rdf:type},C),X) \leftarrow t(C,\textit{owl:intersectionOf},(C_1,C_2)) \wedge  \\
\textit{holds}(\tf(S,\textit{rdf:type},C_1),X) \wedge  \textit{holds}(\tf(S,\textit{rdf:type},C_2),X) $ 

 
    
  \noindent \hangindent=0.6cm $6. \ \textit{holds}(\tf(S,\textit{rdf:type},C),X) \leftarrow t(S,\textit{owl:someValuesFrom,R)} \wedge t(R,\textit{owl:onProperty},P) \wedge 
  	holds(and(\tf(S,P,I) , \tf(I,\textit{rdf:type},R)),X)$
  
\noindent \hangindent=0.6cm $7. \ \textit{holds}(\textit{false},X) \leftarrow \textit{holds}(\tf(I_1,\textit{rdf:type},A),X) \wedge 
 \textit{holds}(\tf(I_2,\textit{rdf:type},B),X) \wedge \\
  t(A,\textit{owl:disjointWith},B) $

\smallskip

\noindent where \textit{false} is a term representing $\bot$.

We denote by $\mathcal A$ the set of rules that encode the terminological and functional annotations, that
is, (1) the OWL assertions of the form $ BpsEl : \SOME{\textit{termRef}}{Concept}$; (2) the facts defining the relations $\textit{pre}(A,C,P)$, $\textit{eff}(A,Q,E^{-},E^{+},P)$, $c\_seq(G,B,Y,P)$; 
(3) the rules for evaluating $\textit{holds}(\tf(s,p,o),X)$ atoms
(such as rules 1--7 above). 

\subsubsection{Change, Ramification and Consistency}

The logical formalization of activity preconditions and effects given above
has to be compared with various solutions to the \textit{Frame}  and 
\textit{Ramification} problems proposed by 
the various AI formalisms for representing action and change. 

The Frame Problem was formulated in \cite{frame-problem}
as the problem of ``expressing a dynamical domain
in logic without explicitly specifying which conditions are not affected by an
action". Basically, it is concerned with \textit{representational} issues, related to the effort needed to specify non-effects of actions, and \textit{inferential} issues, related to the
effort needed to actually compute these non-effects. 

The Fluent Calculus addresses both the representational and inferential aspects of the Frame Problem \cite{fluent-frame} by
modeling change as the difference between two states, caused by actions that deterministically result in a bounded number of direct (positive and negative) effects.  These effects are captured by
 \textit{state update axioms}  specifying the fluents that are added or removed from a state. 
The rules defining the \textit{result} relation introduced in Section \ref{section:trace-theory} can be viewed as a specialized form of state update axioms.

The Ramification Problem \cite{ramification} is the problem of representing and inferring information about indirect effects of actions. Indirect
effects are not explicitly represented in action specifications, but follow from general laws (domain axioms) describing dependencies among fluents. In our framework, general laws are specified in the reference ontology TBox, whose axioms, as discussed in the previous section, are used in the derivation of additional $\tf$ fluents from those belonging to a given state. 
Indirect effects may lead to undesired consequences when 
performing state update.
 For instance, let us consider the fluent  
$f=\tf(o,\textit{rdf:type},\textit{bro:FulfilledPO})$. If we consider the ontology $\O$ given in Table \ref{tab:reference-onto}, we can  infer that $f$ holds in a state $S$ which contains 
$\tf(o,\textit{rdf:type},\textit{bro:Purchase\_Order})$, $\tf(o,\textit{bro:related},i)$, and $\tf(i,\textit{rdf:type},\textit{bro:Invoice})$. Now, assume that the set of
negative effects of the subsequent activity $a$ includes the fluent $f$. 
Then, after the state update determined by $a$, $f$ still holds,
in contrast with the intended meaning of negative effects. 

Many approaches have been proposed to handle such a situation. Some of them are based on the computation of all the possible states $s_i$ caused by the execution of action $a$ in state $s$, such that: \textit{i)} they comply with the domain axioms and the negative effects of $a$, \textit{ii)} they differ minimally from $s$ (see, e.g., the Possible Model Approach - PMA \cite{pma}). This approach introduces a nondeterministic behavior in the state update that appears to be in contrast with the strong prescriptive nature of procedural BP models. Considering again the example above, the execution of $a$ according to the PMA would result in three states: $s - \{\tf(o,\textit{rdf:type},\textit{bro:Order}) \}$, $s - \{\tf(o,\textit{bro:related},i) \}$, and $s - \{\tf(i,\textit{rdf:type},\textit{bro:Invoice}) \}$.

Another solution proposed in the context of the Fluent Calculus, is based on \textit{causal propagations} regulated by \textit{causal relationships} \cite{causal-dependencies}, which specify how indirect effects are derived from direct effects and domain axioms. Causal relationships are then included in the state update axioms and applied until a fix-point is reached. This approach requires an additional formalism for the definition of causal relationships, and the burden for users of providing  additional domain-dependent assertions, which cannot be represented within the ontology. 

Here we follow a different approach based on the following 
{\em consistency condition},
which has to be enforced by every reachable state of a BPS: (i) no contradiction can be derived from the fluents belonging to the state
by using the state independent axioms of the reference ontology,
and (ii) no negative effect of an activity holds after its execution.
Formally, we say that $\textit{eff}$ is {\em consistent} with process $P$ if,
for every flow element $A$ and states $S_1,S_2,$ 
the following implication is true:

\smallskip
\noindent
{\em If}  $S_1$ is reachable from the initial state of $P$
and the relations
$\textit{result}(S_1,\textit{complete}(A),S_2)$ and $\textit{eff}(A,E^{-},E^{+},P)$ hold,
\\
{\em Then}~$\mathcal O \cup\mathcal A \cup 
\{\neg\textit{holds}(\textit{false}, S_2)\}$ is consistent
and, for all $F\in E^{-}$, $\mathcal O \cup \mathcal A \cup \{\neg \textit{holds}(F, S_2)\}$ is consistent.

\smallskip

This condition takes into account that,  since $\mathcal O \cup\mathcal A$
is a definite logic program, 
it only allows the derivation of \textit{positive} indirect effects,
and thus,  for all $F\in E^{+}$, $\mathcal O \cup \mathcal A \cup \{\textit{holds}(F, S_2)\}$ is consistent. 
We will show in Section \ref{section:reasoning} how the
consistency condition can be checked by using the rule-based
temporal logic we will present in the next section. 

From a pragmatic perspective, the modeler is asked to refine the annotation of a BPS until a consistent description of the effects is achieved, possibly disambiguating the situations where underspecified effects may lead to hidden flaws.   

\section{Temporal Reasoning} \label{section:temporal-reasoning}

In order to provide a general verification mechanism for behavioral properties, 
in this section we propose a model checking methodology based on a
formalization of the temporal logic 
CTL ({\em Computation Tree Logic}, see \cite{clarke} for a comprehensive overview)
as a set of rules. 
Model checking is a widely accepted technique for the 
formal verification of BP schemas, as their execution semantics  
is usually defined in terms of states and state transitions, and hence
the use of temporal logics for the specification and verification of properties is a very
natural choice~\cite{bpel-mc,compliance-ibm}. 

CTL is a propositional
temporal logic introduced for reasoning about the behavior of 
reactive systems. 
The behavior is represented as the tree of
states that the system can reach, and each path of this tree is
called a \emph{computation path}. CTL formulas are built from: the constants \textit{true} and \textit{false}; a given set \( \mathit{Elem} \) of \emph{elementary properties}; the connectives: \( \neg \) (`not') and \( \wedge  \) (`and'); the  linear-time operators along a computation path: \textbf{G} 
       (`globally' or `always'), \textbf{F}
       (`finally' or `sometimes'), \textbf{X} (`next-time'), and \textbf{U} (`until'); the quantifiers over computation paths: \textbf{A} (`for
all paths') and \textbf{E} (`for some path'). The abstract syntax of  CTL is defined as follows. 

\begin{definition} [CTL formulas] \label{def:ctl} A CTL formula \( F  \) has
the following syntax:\rm

\smallskip

\makebox[10mm]{$F \, ::=$}$e\, \, \, |\, \, \, \, \textit{true}\, \, \, |\, \, \, \,
\textit{false}\, \, \, |\, \, \, \,
\neg F \, \, \, |\, \, \, \, F _{1}\, \wedge \, F _{2}
\, \, \, |\, \, \, \textbf{ EX}(F )\, \, \, |\, \, \,  \textbf{EU}(F _{1},F _{2})\, \, \, |\, \, \, \textbf{EG}(F )$

\smallskip

\noindent \em where \( e \) belongs to a given set \( \mathit{Elem} \) 
of \emph{elementary properties}. 
\end{definition}
\noindent Other operators can be defined in terms of the ones given 
in Definition \ref{def:ctl}, e.g.,
\( \textbf{EF}(F)\, \equiv \, \textbf{EU}(\mathit{true},F ) \) and
\( \textbf{AG}(F )\, \equiv \, \neg \textbf{EF}(\neg F ) \)  \cite{clarke}.

Usually, the semantics of CTL formulas is defined by introducing a
\emph{Kripke structure} \( \mathcal{K} \), which represents the state space and 
the state transition relation, and by defining the
satisfaction relation \( \mathcal{K},s\models F  \), which
denotes that a formula \( F  \) holds in a state \( s \) of \(
\mathcal{K} \)~\cite{clarke}. 
In order to verify temporal properties of the behavior of a BPS $P$,
we define a Kripke structure associated with $P$. The states are defined as
finite sets of ground fluents and the state transition relation is based on 
the immediate reachability relation $r$ between states defined in Section~\ref{section:trace-theory}.
The Kripke structure and the satisfaction relation will be encoded by sets of rules, hence providing a uniform framework for reasoning about the ontological properties and the behavioral properties of business processes.

A Kripke structure is a four-tuple $\mathcal K = \langle \mathcal S,\mathcal I,\mathcal R, \mathcal L\rangle$ 
defined as follows. 
\begin{enumerate}

\item  $\mathcal S$  is the finite set of all states, where a state is a finite 
set of ground fluents. 

\item $\mathcal I$ is the {\em initial state} of BPS $P$,  encoded by the rule:

\smallskip

$\textit{initial}(I,P) \leftarrow\textit{bp}(P,S,E) \wedge  I = \{\textit{cf}(\textit{start},S,P)\} $
\smallskip

\item $\mathcal R$ is the {\em transition relation}, which is
defined as follows: $\mathcal  R(X,Y)$ holds iff $r(X,Y)$ holds, where
$r$ is the predicate defined in Section \ref{section:trace-theory},
i.e., $\mathcal  R(X,Y)$ holds iff 
there exists an action $A$ that can be executed in state $X$ leading to state $Y$.  

\item $\mathcal L$ is the {\em  labeling function}, which associates with each state $X$ 
the set of fluents $F$ such that $\mathcal O \cup \mathcal A \models \textit{holds}(F,X)$.
\end{enumerate}

\noindent
In the  definition of Kripke structure given in \cite{clarke}, the transition relation $\mathcal R$ is assumed to be {\em  total}, that is,
every state $S_1$ has at least one {\em successor state} $S_2$ for which 
$\mathcal R(S_1,S_2)$ holds. 
This assumption is justified by the fact that reactive systems can be
thought as ever running processes.
However, this assumption is not realistic in the case of business processes,
for which there is always at least one state with no successors, namely 
one where the \textit{end} event of a BPS has been completed. 
For this reason the semantics of the temporal operators given in \cite{clarke}, 
which refers to
\textit{infinite} paths of the Kripke structure, is suitably changed here, according to \cite{finite-transition-systems}, by taking into 
consideration \textit{maximal} paths, i.e., paths that are either infinite or end with a state
that has no successors, called a \textit{sink}. 

\begin{definition}[Maximal Path]
A {\em maximal path} in \( \mathcal{K} \) starting
from a state \( S_{0} \) is either 

\vspace*{-2mm}

\begin{itemize}
	\item an \emph{infinite} sequence of states \(S_{0}\, S_{1}\ldots \,  \) such that \( S_{i}\, \mathcal R\, S_{i+1} \), for every \( i\! \geq \! 0 \); or

\vspace*{-2mm}

	\item a \emph{finite} sequence of states $S_0\, S_1 \ \ldots \ S_k$, with $k \geq 0$, such that:	
\begin{enumerate}

\vspace*{-1mm}

	\item \( S_{i}\, \mathcal R\, S_{i+1} \), for every  $0 \leq i < k$, and
	\item there exists no state $S_{k+1} \in \mathcal S$ such that \( S_{k}\, \mathcal R\, S_{k+1} \).
\end{enumerate}
\end{itemize}

\vspace*{-2mm}

\end{definition}

\vspace*{-2mm}

\noindent
The semantics of CTL operators can be encoded by extending 
the definition of the predicate $\textit{holds}$. Below
we list the semantics of those operators
and the corresponding rule-based formalization. 

\smallskip

\noindent $\textbf{EX} (F)$ holds in state $S_0$ if \(F \) holds in a successor state of $S_0$:

\smallskip

\noindent \( \textit{holds}(\ex (F),S_0)\leftarrow \textit{r}(S_0,S_1) \wedge \textit{holds}(F,S_1) \)

\smallskip

\noindent $\textbf{EU} (F_1,F_2)$ holds in state $S_0$ if there
exists a maximal path \(\pi\): \( S_0\;S_1\ldots \,  \) 
such that for some \(S_n\) in \(\pi\) we have that $F _{2}$
holds in $S_n$ and,  for \( j =0,\ldots ,n\! -\! 1\),  $F _{1} $
holds in~$S_{j}$:

\smallskip

\noindent \( \textit{holds}(\eu(F_{1},F_{2}),S_0)\leftarrow \textit{holds}(F_{2},S_0) \)

 \noindent \hangindent=0.3cm \( \textit{holds}(\eu (F_{1},F_{2}),S_0) \leftarrow \textit{holds}(F_{1},S_0)\wedge \textit{r}(S_0,S_1) \wedge
\textit{holds}(\eu (F_{1},F_{2}),S_1) \)

\smallskip

\noindent $\textbf{EG} (F)$ holds in a state $S_0$ if there exists a 
maximal path $\pi$ starting from $S_0$ 
such that $F$ holds in each state of $\pi$.
Since the set of states is finite, $\textbf{EG}(F)$ holds in $S_0$ if
there exists a finite path $S_0 \ \ldots \ S_k$ such that, for $i=0,\ldots,k$,
$F$ holds in $S_i$, and either (1) $S_j=S_k$, for some $0\leq j < k$, or
(2) $S_k$ is  a sink state. 
Thus, the semantics of the operator $\textbf{EG}$ is encoded
by the following rules:

\smallskip

\noindent \hangindent=0.3cm \( \textit{holds}(\eg (F),S_0)\leftarrow \textit{fpath}(F,S_0,S_0) \)

\noindent \hangindent=0.3cm \( \textit{holds}(\eg (F),S_0)\leftarrow \textit{holds}(F,S_0) 
\wedge \textit{r}(S_0,S_1) \wedge \textit{holds}(\eg(F),S_1) \)

\noindent \hangindent=0.3cm \( \textit{holds}(\eg (F),S_0)\leftarrow \textit{sink}(S_0) \wedge \textit{holds}(F,S_0) \)

\smallskip

\noindent where: (i) the predicate  $\textit{fpath}(F,X,X)$ holds if there
exists a path from $X$ to $X$ itself, consisting of at least one $r$ arc, such that $F$ holds in every state on the path:

\smallskip

\noindent
\noindent \hangindent=0.3cm $\textit{fpath}(F,X,Y) \leftarrow
\textit{holds}(F,X) \wedge \textit{r}(X,Y)$

\noindent
\noindent \hangindent=0.3cm \( \textit{fpath}(F,X,Z)\leftarrow  \textit{holds}(F,X) \wedge 
\textit{r}(X,Y) \wedge \textit{ fpath}(F,Y,Z) \)

\smallskip
\noindent and (ii) the predicate $ \textit{sink}(X)$ holds if
$X$ has no successor state.

%
%
%

Finally,  the following rules define the properties characterizing the initial and the final state of a process:


\smallskip


\noindent   \( \textit{holds}(F,s_0(P))\leftarrow \textit{initial}(I,P) \wedge  \textit{holds}(F,I) \)

\noindent
 \( \textit{holds}(\textit{final}(P),X)\leftarrow \textit{bp}(P,S,E) 
\wedge  \textit{holds}(\textit{cf}(E,\textit{end},P),X) \)

\smallskip

\noindent
The rules defining the semantics of the operator $\textbf{EG}$ are similar to
the constraint logic programming definition proposed in \cite{lp-mc-cl2000}.
However, as already mentioned, in this paper we refer to the notion of
maximal path instead of infinite path.
Similarly to  \cite{lp-mc-cl2000}, our definition of the semantics
of $\textbf{EG}$ avoids the introduction of
greatest fixed points of operators on sets of states which are often
required by the  approach described in \cite{clarke}.
Indeed, the rules defining $\textit{holds}(\eg (F),S_0)$ are interpreted
according to the usual least fixpoint semantics (i.e.,  the
least Herbrand model \cite{lloyd}).

The encoding of the satisfaction relation for other CTL operators, e.g, $\textbf{EF}$ and
$\textbf{AG}$, follows from the  equivalences defining them \cite{clarke}. It is worth noting that in some special cases the assumption that paths are maximal, but not necessarily infinite, matters \cite{finite-transition-systems}. For instance,
if $S_0$ is a sink state, then $\textit{holds}(\textit{ag} (F),S_0)$ is true iff $\textit{holds}(F,S_0)$ is true, 
since the only maximal path starting from $S_0$ is the one constituted by $S_0$ only. 
Finally, we would like to note that the definition of the CTL semantics
given here is equivalent to the one in 
\cite{clarke} in the presence of infinite computation paths only. 

\section{Reasoning Services}
\label{section:reasoning}

Our rule-based framework supports several
reasoning services that can combine complex knowledge
about business processes from different perspectives, such as
the workflow structure, the ontological description, and the behavioral semantics.
In this section we will illustrate three such services:
verification, querying, and trace compliance checking.

Let us consider the following sets of rules:
(1)~$\mathcal {B}$, representing a set of BP schemas and the BP 
meta-model defined in Section \ref{section:bpal},
(2) $\mathcal T$,  defining the behavioral semantics presented in 
Section \ref{section:trace-theory},
(3) $\mathcal O$, collecting the OWL triples and rules
that represent the business reference ontology
defined in Section \ref{sec:enrichment}, 
(4) $\mathcal A$, encoding the annotations defined
in Sections~\ref{sec:TermAnn}
and~\ref{section:state-updates}, 
and (5) $\CTL$, 
defining the semantics of CTL presented in Section~\ref{section:temporal-reasoning}.

Let $\KB$ be the set of rules $\mathcal B \ \cup \ \mathcal T \ \cup 
\ \mathcal O\ \cup \ \mathcal A\ \cup \ \CTL$.  $\KB$ is called a
{\em Business Process Knowledge Base} (BPKB). 
It is straightforward to show that
$\KB$ is stratified, and hence its semantics is unambiguously 
defined by its perfect model $\textit{Perf}(\KB)$~(see Section \ref{subsec:lp}).

\subsection{Verification}
\label{section:verification}

In the following we present some examples of properties
that can be specified and verified in our framework.
A property is specified by a predicate $\textit{prop}$ defined by a rule $C$ in terms
of the predicates defined in $\KB$. The verification task is performed by
checking whether or not $\textit{prop} \in \textit{Perf}(\KB \cup \{C\})$.

\smallskip

\noindent (1) A very relevant behavioral property of a BP $p$ is that from any  
reachable state, it is possible to complete the process, i.e., reach the final state. 
This property, also known as \textit{option to complete} 
\cite{vander-workflow}, can be specified by the following rule, stating that
the property $\textit{opt\_com}$ holds if the 
CTL property $\textbf{AG}(\textbf{EF}(\textit{final}(p)))$ 
holds in the initial state of $p$:

\smallskip
\noindent \hangindent=3mm  $\textit{opt\_com}\leftarrow \textit{holds}(\ag(\ef(\textit{final}(p))),s_0(p))$



\smallskip



\noindent (2) Temporal queries allow us to verify the consistency 
condition for effects introduced in Section \ref{section:state-updates}. 
In particular, given a BPS $p$, 
inconsistencies due to the violation of some integrity constraint 
defined in the ontology by rules of the form
$\bot \leftarrow G$ (e.g., concept disjointness) 
can be verified by defining the \textit{inconsistency} property as follows:

\smallskip

\noindent \hangindent=3mm  $\textit{inconsistency}\leftarrow \textit{holds}(\ef(\textit{false}),s_0(p))$ 

\smallskip

\noindent (3) Another relevant property of a BPS is \textit{executability}
\cite{beyond-soundness},  according to which no activity reached by the control flow 
should be unable to execute due to some unsatisfied enabling condition. In our 
framework we can specify non-executability by defining a predicate \textit{n\_exec}
which holds if it can be reached a state where some activity $A$ is waiting for 
execution but is not possible to start its enactment. 

\smallskip
\noindent \hangindent=3mm  $\textit{n\_exec}\leftarrow \textit{holds}(\ef(\textit{and}(\textit{cf}(A1,A,p), \textit{not}(\ex(\textit{en}(A,p))))),s_0(p)) \wedge \textit{activity}(A)$ 

\smallskip

\noindent \hangindent=0cm (4) Temporal queries can also be used for the verification of 
\textit{compliance rules}, i.e.,  directives expressing internal policies and regulations 
aimed at specifying the way an enterprise operates. In our Handle Order example, one such compliance rule may be that every \textit{order} is eventually \textit{closed}. 
In order to verify whether this property holds or not, we can define a
\textit{noncompliance} property which holds if 
it is possible to reach the final state of the process where, 
for some $O$, it can be inferred that
$O$ is an \textit{order} which is not \textit{closed}. In our 
example \textit{noncompliance} is satisfied, and thus the compliance rule is not enforced. 
In particular, if the exception attached to the
\textit{accept order} task is triggered, the enactment continues with the \textit{notify 
rejection}  task (due to the guards associated to $g1$), and the order is never 
\textit{canceled} nor \textit{fulfilled}. 

\smallskip
\noindent \hangindent=3mm  $\textit{noncompliance}\leftarrow  \textit{holds}(
	\ef(
		\textit{and}(
			\tf(O,\textit{rdf:type},\textit{bro:Purchase\_Order}),\\ \textit{and}(
				\textit{not}(
					\tf(O,\textit{rdf:type},\textit{bro:ClosedPO})
				),\textit{final}(p)
			)
		)
	,s_0(ho))$

\smallskip

\subsection{Retrieval}
\label{section:query}
The inference mechanism based on SLG-resolution can be
used for computing boolean answers to ground queries, but also
for computing, via unification, substitutions for variables occurring in 
non-ground queries.
By exploiting this query answering mechanism we can easily provide,
besides the verification service described in the previous section, 
also reasoning services for the retrieval of process fragments.

The following queries show how process fragments can be retrieved according to
different criteria. For sake of readability, we introduce the relation $\sigma(A,C)$ as an abbreviations for the OWL expression $A:\SOME{\textit{termRef}}{C}$ encoding terminological annotations.

Query $q_1$ computes every activity $A$ 
performed by a $\textit{Carrier}$ and realizing a \textit{Transportation} (e.g., \textit{delivering}) ;
$q_2$ computes every decision point (exclusive branch) $G$ occurring along a path of a BPS $P$ delimited by two 
activities $A$ and $B$, where the former operates on \textit{orders} (e.g., \textit{create\_order})  and the latter  is included in the results of $q_1$; finally, $q_3$ retrieve all the activities operating on \textit{orders}  which  precede  (in every possible execution) a \textit{Transportation} performed by a $\textit{Carrier}$ (e.g., \textit{create\_order}).

\smallskip

\noindent \hangindent=3mm $q_1(A) \leftarrow activity(A) \wedge assigned(A,C,P) \wedge \sigma(C,\textit{bro:Carrier}) \wedge \sigma(A,\textit{bro:Transportation})$

\smallskip

\noindent \hangindent=3mm $q_2(A,G,B,P) \leftarrow q_1(B) \wedge output(A,I,P) \wedge \sigma(I,\textit{bro:Purchase\_Order}) \wedge reachable(A,G,P) \wedge reachable(G,B,P)$

\noindent \hangindent=3mm $q_3(A,B,P) \leftarrow q_1(B) \wedge output(A,I,P) \wedge \sigma(I,\textit{bro:Purchase\_Order}) \wedge  reachable(A,B,P) \wedge \textit{holds} (not(eu(not(en(A,P)),en(B,P))),s_0(P))$

\subsection{Trace Compliance}

The execution of a process is modeled as an \textit{execution trace}
(corresponding to a  plan in the Fluent Calculus), i.e.,  a sequence of actions of the form $[\textit{act}(a_1), \ldots, \textit{act}(a_n)]$ where \textit{act} is either \textit{begin} or \textit{complete}. The 
predicate  $\textit{trace}(S_1,\textit{T},S_2)$ defined below
holds if \textit{T} is a sequence of actions that lead from state $S_1$
to state $S_2$:

\smallskip

\noindent
$\textit{trace}(S_1,[~],S_2) \leftarrow S_1 = S_2$

\noindent
\hangindent=3mm$\textit{trace}(S_1,[A| \textit{T}],S_2)\! \leftarrow\!
\textit{result}(S_1,A,U) \!\wedge\!
\textit{trace}(U\!,\textit{T},S_2)$

\smallskip

 A \emph{correct trace} $T$ of a BPS $P$ is a trace that
leads from the initial state to the final state of $P$, that is:

\smallskip

\noindent \hangindent=3mm $\textit{ctrace}(T,P) \leftarrow \textit{initial}(I,P) \wedge \textit{trace}(I,T,Z) \wedge \textit{holds}(\textit{final}(P),Z)$

 \smallskip

Execution traces are commonly stored by BPM systems as process logs, 
representing the evolution of the BP instances that have been enacted.
The correctness of a trace $t$ with respect to a
given BPS $p$ can be verified by evaluating a query of the form 
$\textit{ctrace}(t,p)$
where $t$ is a ground list and $p$ is a process name.

The rules defining the predicate \textit{ctrace} can also be used to {\em generate}
the correct traces of a process $p$ that satisfy some given property. 
This task is performed by evaluating a query of the form 
$\textit{ctrace}(T,p) \wedge \textit{cond}(T)$,
where $T$ 
is a free variable and $\textit{cond}(T)$ is a property that 
$T$ must enforce. For instance, we may want to generate traces where the execution 
of a flow element  $a$ is followed by the execution of a flow element $b$:
   
\smallskip

\noindent \hangindent=3mm $\textit{cond}(T) \leftarrow \textit{concat}(T_1,T_2,T) 
\wedge \textit{complete}(a) \in T_1 \wedge \textit{complete}(b) \in T_2 $

\section{Computational Properties}
\label{sect:proof}

In this section we prove the
soundness, completeness, and termination of query evaluation
using SLG-resolution.
We also provide an upper bound to the time complexity
of query evaluation.

As mentioned in Section \ref{subsec:lp}, the
soundness and completeness of SLG-resolution with respect to
the perfect model semantics is guaranteed for
the class of queries that do not flounder.
In \cite{lloyd} a sufficient condition ensuring
that a query does not flounder is based on the notion
of {\em allowed} query and  rule.
In particular, a query is allowed if every variable occurring in it
also occurs in one of its positive literals.
Similarly, a rule is allowed
if every variable occurring in it
also occurs in a positive literal in
its body. 
Unfortunately, not all rules in $\KB$
are allowed in the sense of \cite{lloyd}. For instance,
the variables $F$ and $S$ occurring in the rule
$ \textit{holds}(not(F),S)\leftarrow \neg \textit{holds}(F,S) $,
do not occur in any positive
literal of the body.

We will now define a subclass of the
allowed queries whose evaluation with respect to 
$\KB$ does not flounder.
The definition of this subclass also takes into account
the left-to-right selection strategy for literals.
For any predicate defined in $\KB$, 
each argument denoting a state (i.e., a set of fluents)
can be classified either as an {\em input argument}
or as an {\em output argument}.
This classification is often called a {\em mode} \cite{Apt97}.
In particular, it can be shown that we can classify the
arguments such that the following property holds:
if a predicate is evaluated with all its input
arguments bound to ground sets of fluents,
then whenever the predicate succeeds 
all variables occurring in output arguments are bound to ground 
sets of fluents.
For instance, for the predicate
$\textit{result}(S_1,A,S_2)$ the first argument 
is an input argument and the
third argument is an output argument.
For reasons of space we do not list here, for
each predicate defined in $\KB$, the input or output 
classification of its arguments.
The following notion is adapted from \cite{Apt97}.

\begin{definition} 
A query $L_1 \wedge \ldots \wedge L_n$ is {\em well-moded} if,
for $i=1,\ldots,n,$ every variable occurring in an input argument
in $L_i$ occurs in an output argument in $L_j$, for some
$j\in \{1,\ldots,i-1\}$.
\end{definition}

\begin{definition} 
Let $f$ be (a term representing) a CTL formula. A subformula $e$ of $f$ is 
{\em grounding} if $e$ is a fluent and
one of the following conditions hold: 
(i) $f$ is a fluent and $e$ is $f$, 
(ii) $f$ is $\textit{and}(f_1,f_2)$ and $e$ is a grounding subformula of either $f_1$ or $f_2$,
(iii) $f$ is $\ex(f_1)$ and $e$ is a grounding subformula of $f_1$,
(iv) $f$ is $\eu(f_1,f_2)$ (or, in particular, $f$ is $\ef(f_2)$) and $e$ is a grounding 
subformula of $f_2$,
(v) $f$ is $\eg(f_1)$ and $e$ is a grounding subformula of $f_1$.
\end{definition}

\begin{definition} 
A query $Q$ of the form
$L_1 \wedge \ldots \wedge L_n$ is an
{\em NF}-query (short for {\em Non-Floundering} query)
if the following conditions hold:

\smallskip

\noindent
(1) $Q$ is well-moded,

\smallskip

\noindent
(2) for $i=1,\ldots,n,$ 
if $L_i$ is of the form $\textit{holds}(f,S)$, then all 
variables of $f$ occur in fluents that are subformulas of $f$, and 

\smallskip

\noindent
(3) for  each variable $X$ of $Q$, the leftmost occurrence $X_l$
of $X$ in $Q$ appears in a positive literal $L_j$, 
with $1\leq j\leq n$, 
such that either (3.1) $L_j$ has  predicate different from
`\textit{holds}'
or  (3.2)~$L_j=\textit{holds}(f,S)$ and $X_l$ 
appears in a grounding subformula of $f$. 

\smallskip

A rule of the form
$A\leftarrow L_1 \wedge \ldots \wedge L_n$
is an {\em NF-rule},  if the following conditions
hold:

\smallskip

\noindent
(4) no variable ranging over states occurs in $A$,

\smallskip

\noindent
(5) every variable occurring in $A$
also occurs in $L_1 \wedge \ldots \wedge L_n$,
and 

\smallskip

\noindent
(6) $L_1 \wedge \ldots \wedge L_n$ is an NF-query.
\end{definition}

\noindent
For example, the queries and rules presented in
Sections \ref{section:verification} and 
\ref{section:query} are all NF.
The query $\textit{holds}(\eu(\textit{en}(A,p), \textit{true}),s_0(p)) \wedge \neg \textit{task}(A)$ is not an
NF-query, because $\textit{en}(A,p)$ is not a grounding subformula of $\eu(\textit{en}(A,p), \textit{true})$.
This query flounders, as the non-ground negative literal
$\neg \textit{task}(A)$ will be selected after the success of
$\textit{holds}(\eu(\textit{en}(A,p), \textit{true}),s_0(p))$.

We assume that the query $Q$ is defined by a single NF-rule
$Q \leftarrow B$, where $B$ is a conjunction of literals.
The extension to the case where $Q$ is defined by a set of
NF-rules (like in Section \ref{section:query})
is straightforward.

\begin{proposition}\label{prop:nf}
Let $Q \leftarrow B$ be an NF-rule such that the predicate 
of $Q$ does not occur in $\KB$. 
Then we have the following properties.

\noindent
(1) The query $Q$  does not flounder with respect to
$\KB \cup \{Q \leftarrow B\}$.

\noindent
(2) Every answer for $Q$ with respect to
$\KB \cup \{Q \leftarrow B\}$ is a ground substitution
for the variables in $Q$.
\end{proposition}

\begin{proof} (Sketch)
(1) Let us consider a one-step derivation
$L\wedge Q_1 \stackrel{\theta}{\longrightarrow}Q_2$.
By cases on the form of $L$ one can show that
if $L\wedge Q_1$ is an NF-query, then $Q_2$
is an NF-query.
Thus,  by also using
the fact that every ground atom is
an NF-query, if $\neg A \wedge Q_n$ is a query
generable from $Q$ in any number of steps, 
then $\neg A \wedge Q_n$ is an NF-query.
Therefore,  all variables
occurring in $\neg A$ must also occur in a positive literal
to the left of $\neg A$, and  hence $\neg A$ is a ground atom.

(2) Suppose, by contradiction, that an answer $\theta$ for $Q$
is not a ground substitution. Let us consider the rule
$Q \leftarrow B\wedge \neg R$, where $R$ is any atom
containing one of the variables that are not bound to
a ground term in $\theta$. 
$Q \leftarrow B\wedge \neg R$ is an NF-rule.
We can construct a derivation from $Q$ that
eventually selects the non-ground literal $\neg R\theta$,
and hence the query $Q$  flounders with respect to
$\KB \cup \{Q \leftarrow B \wedge \neg R\}$.
\end{proof}

Let us now show that the evaluation of every NF-query
terminates by using SLG resolution.
Given an atomic query $Q$, we define:

\begin{itemize}

\item $\textit{Calls}_Q$ as the least set of atoms satisfying the following properties:

\smallskip

\makebox[6mm][l]{(1)}$Q \in \textit{Calls}_Q$;

\makebox[6mm][l]{(2)}if $A\in \textit{Calls}_Q$ and 
either
$A \stackrel{\theta_1}{\longrightarrow} \ldots \stackrel{\theta_n}{\longrightarrow}  A' \wedge Q'$
or $A \stackrel{\theta_1}{\longrightarrow} \ldots \stackrel{\theta_n}{\longrightarrow}  \neg A' \wedge Q'$,

\makebox[6mm][l]{}then $A'\in \textit{Calls}_Q$;

\item $\textit{Answers}_Q$
as the set of atoms $A \theta$ such that $A\in \textit{Calls}_Q$ and
$\theta$ is an answer for $A$;

\item $\Delta_Q$ as $\textit{Calls}_Q\cup \textit{Answers}_Q$.

\end{itemize}

\noindent
The termination proof is based on the property that, 
for any query $Q$, 
$\Delta_Q$ is a finite set of atoms.
This property is equivalent to 
the {\em bounded-term-size property} that in~\cite{tabling}
has been shown to be a sufficient condition for termination of 
SLG-resolution~\cite{tabling}.

Given a set $S$, by $|S|$ we denote the cardinality of $S$.
Let $P$ be a logic program, by $\Pi_P$ we
denote the maximum number of literals in the body of a
rule in $P$. 
The following result is an adaptation of Theorem 5.4.3
in~\cite{tabling}.

\begin{theorem}[Termination of SLG-resolution] 
\label{th:termination}
Let $P$ be a logic program and $Q$ be an atomic query.
Suppose that there exists a finite set $\mathcal D$ of atoms such
that $\Delta_Q\subseteq \mathcal D$. 
Then all answers for $Q$ can be 
computed by SLG-resolution in 
$\mathcal O(|P|
\times |\mathcal D|^{\Pi_P+1})$ steps.
\end{theorem}

By applying Theorem~\ref{th:termination}
to the case where $P$ is of the form
$\KB\cup \{q(X) \leftarrow B\}$,
we get the following result.

\begin{proposition}\label{prop:complexity}
Suppose that $q(X) \leftarrow B$ is an NF-rule, where $X$ is a tuple
of $k\geq 0$ variables and the predicates of $B$ are defined in
$\KB$. Then, all answers for $q(X)$ can be 
computed by SLG-resolution in 
$\mathcal O(|\mathcal {KB}|
\times (|\mathcal F|^k +
(|\!| B|\!| \times |\mathcal F|^v \times|\mathcal S|) +
|\mathcal S|^m)^{r+1})$ steps,
where: 
(i) $\mathcal F$ is the set of ground fluents that can be 
defined in $\KB$,
(ii) $\mathcal S$ is the set of possible states,
that is, the powerset of $\mathcal F$,
(iii) $|\!| B|\!|$ denotes the size 
(that is, the number of symbols) of $B$, 
(iv) $v$ is the largest
number of variables in a CTL formula in $B$, 
(v) $m$ is the largest arity of a predicate in
$\KB$, and 
(vi) $r$ is the largest number of literals
in the body of a rule in $\KB\cup \{q(X) \leftarrow B\}$.
\end{proposition}

\begin{proof}
Suppose that $Q$ is the query $q(X)$ defined
by the NF-rule  $q(X) \leftarrow B$, where $X$ is tuple
of $k\geq 0$ variables and the predicates of $B$ are defined in
$\KB$.
Let us define the following set $\mathcal D$ of atoms, 
where $\mathcal V$
is a finite, sufficiently large set of variables, and
$\mathcal E$ is the set of flow elements in $\KB$.

\smallskip

\noindent
\makebox[7mm][l]{$\mathcal D=$}$\{q(t) \mid t\in 
(\mathcal E \cup \mathcal F \cup \mathcal V)^k \}\ \cup$

\smallskip

\noindent
\makebox[7mm][l]{}$\{\textit{holds}(f,s) \mid  f=f'\theta$,
for some CTL-formula $f'$ occurring as a subformula in a literal\linebreak
\makebox[34mm][l]{}of $B$ and substitution $\theta$ from variables to fluents, and
$s\in \mathcal S\}\ \cup$

\smallskip

\noindent
\makebox[7mm][l]{}$ \{ p(u) \mid p$ ($\neq \textit{holds}$) is an $m$-ary predicate  defined in $\KB$
and $u\in (\mathcal E \cup \mathcal F \mathcal \, \cup \, 
\mathcal S\ \cup\ \mathcal V)^m\}$

\smallskip

\noindent
Additionally, we assume that no two atoms in $\mathcal D$
are variants of each other.

We have that
$|\mathcal D| \leq (|\mathcal E| + |\mathcal F| + 1)^k +
(|\!| B|\!| \times |\mathcal F|^v \times|\mathcal S|) +
(|\mathcal E| + |\mathcal F| + |\mathcal S| + 1)^m $.
The fluents in $\mathcal F$ are defined by using the
elements  in $\mathcal E$, 
the constants from the ontology 
(which also occur in $\KB$), and
the function symbols \textit{cf}, \textit{en}, $t_f$,
and \textit{wrtn}, and hence $|\mathcal E|\leq |\mathcal F|$.
Moreover, $|\mathcal S|=2^{|\mathcal F|}$.
Thus, $|\mathcal D| \in \mathcal O(|\mathcal F|^k +
(|\!| B|\!| \times |\mathcal F|^v \times|\mathcal S|) +
|\mathcal S|^m) $. By Theorem \ref{th:termination}, 
we get the thesis.
\end{proof}

By using Propositions \ref{prop:nf} and \ref{prop:complexity}, 
we get the following result.

\begin{theorem}[Termination, Soundness, and Completeness of Query Evaluation in $\KB$]\label{thm:correctness}
Let $Q \leftarrow B$ be an NF-rule such that the predicate 
of $Q$ does not occur in $\KB$. 
Then: 

\noindent
(1)~the evaluation of $Q$ with respect to 
$\KB \cup \mathcal \{Q \leftarrow B\}$ using SLG-resolution
terminates;

\noindent
(2) $Q$ succeeds with answer $\theta$ iff 
$Q\theta\in\textit{Perf}(\KB\cup\mathcal \{Q \leftarrow B\})$;

\noindent
(3) for a ground rule of the form
$\textit{prop} \leftarrow \textit{holds}(f,s)$,
the evaluation of $\textit{prop}$ by using SLG-resolution terminates
in polynomial time in $|\!| f|\!| \times |\mathcal S|$.
\end{theorem}

\begin{proof}
(1) The termination of query evaluation has been proved in Proposition \ref{prop:complexity}.

\smallskip

\noindent
(2) The soundness and completeness of query evaluation
follows from Proposition \ref{prop:nf} and
from the soundness
and completeness of SLG-resolution for non-floundering
queries recalled in Section \ref{subsec:lp}.

\smallskip

\noindent
(3)  If we consider the ground rule 
$\textit{prop} \leftarrow \textit{holds}(f,s)$, then
in Proposition \ref{prop:complexity} we have $k=v=0$.
Since $m$ and $r$ do not depend on $f$ or $\mathcal S$,
we get the thesis.
\end{proof}

Proposition \ref{prop:complexity} above only gives a
loose upper bound on the complexity of query evaluation.
However, it is sufficient for showing that, in line with the complexity
of the CTL verification problem \cite{clarke}, 
our verification method
has polynomial running time with respect to the
number of  states that are potentially reachable during 
process enactment.
Moreover, Theorem \ref{thm:correctness}
shows that the use of  OWL 2 RL
elementary properties does not add more
than polynomial complexity.
A tighter complexity analysis could be
done by directly analyzing the evaluation
of queries with respect to $\KB$, instead of relying,
as done above, on the general results provided by \cite{tabling}.

In practice, our fluent-based representation
of the behavioral semantics determines a running time which is
polynomial in the number of flow elements that are concurrently
enacted plus the number of fluents that are added to states by
functional annotations. Usually, this number is much smaller
than the cardinality of the powerset of $\mathcal F$.
Indeed, the experimental results reported in Section
\ref{sect:experimentation}
show that verification and querying are feasible 
for medium sized, non-trivial processes.

The termination of trace correctness checking can be
proved under assumptions similar to the ones of Theorem \ref{thm:correctness}.
However, stronger assumptions are needed for the termination of
trace generation in the case where we want to compute the set of
{\em all} correct traces satisfying a given condition, as this set may
be infinite in the presence of cycles.

\section{Implementation}
\label{sect:implementation}
In the following we describe the BPAL Platform, a prototypical implementation of the
framework discussed so far (Section \ref{sect:tool}), and we then discuss  an experimental evaluation
of the reasoner performances (Section \ref{sect:experimentation}). 

\subsection{Tool Description}
\label{sect:tool}

The BPAL platform\footnote{A video demonstration is available at \url{http://www.youtube.com/watch?v=xQkapzjhO7g}} is implemented as an Eclipse Plug-in\footnote{\url{http://www.eclipse.org/}}, whose main components are depicted in the functional view in Figure \ref{fig:functional-view}.  It provides the
{\em BPKB Editor} to assist the user through a graphical interface in the definition of a
BPKB, and the {\em BPAL Reasoner},
based on an LP engine, able to operate on the BPKB through the query language {\em QuBPAL}, designed for interrogating a repository of semantically enriched BPs. 
\vspace{-0.2cm}
\begin{center}
\begin{figure}[htbp!]	
\centering
  \includegraphics[width=10cm]{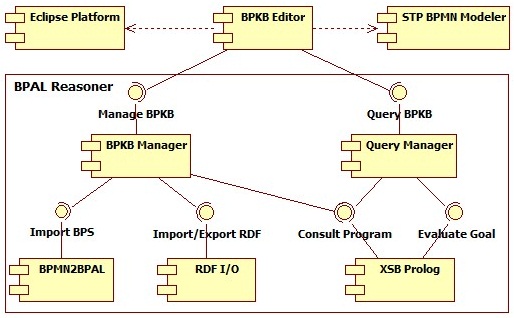}
	\caption{Functional view of the BPAL platform}
	\label{fig:functional-view}
\end{figure}	
\end{center}

\vspace{-1.2cm}
\subsubsection{Query Language} 

QuBPAL is an expressive query language for a BPKB based on the  theoretical framework presented in this paper (a preliminary specification has been discussed in \cite{bpal-bsme}). 
It does not require the user to understand the technicalities of the underlying LP platform,  since QuBPAL queries are SELECT-WHERE statements intended to be automatically translated to logic programs, and then evaluated by using standard LP engines.

The SELECT statement defines the output of the query evaluation, which can be a boolean value, variables occurring in the WHERE statement, and a process selector representing either a BPS or a BPS fragment.  The WHERE statement specifies an expression that restricts the set of data returned by the query, built from the set of the predicates defined in the BPKB (including CTL operators) and  the connectives AND, OR, NOT, and the predicate = with the standard logic semantics. In the queries we use question mark to denote variables (e.g., \textit{?x}), and we
use the notation \textit{?x::C} to indicate the terminological annotation of a variable, i.e., $x:\SOME{\textit{termRef}}{C}$. 

It is worth noting that the representation of OWL/RDFS resources as sets of triples, which directly encode the underlining RDF graph, allows us to pose queries over the ontology in a form very close to the SPARQL (SPARQL Protocol and RDF Query Language) standard \cite{sparql}, defined by the World Wide Web
Consortium and widely accepted in the Semantic Web community. SPARQL is in fact designed to query RDF resources, that essentially are organized as directed and labeled graphs, by matching graph pattern over RDF graphs. Graph patterns are in turn specified as triples where variables can occur in every position (i.e., atoms of the form $t(a_1,a_2,a_3)$), along with their conjunctions and disjunctions. In this sense, while providing additional \textit{primitives} to be used specifically for querying BPs, the ontology-related reasoning is specified in a  QuBPAL query accordingly to consolidated Semantic Web standards.   

To provide some insights about the language, we report in the following two examples of QuBPAL queries.  The first one  represents the formulation of the verification criteria for the compliance rule discussed at Point (4) of Section \ref{section:verification}. The second one is the QuBPAL translation of the query $q_3$ discussed in Section \ref{section:query}.

\smallskip

$\mathtt{SELECT\ <>}$

\hspace*{6mm}
$\mathtt{WHERE\ [EF\ (final(ho)\ AND\ t(?o,rdf:type,bro:PurchaseOrder)\  AND\ NOT}$ \\  
\hspace*{11mm}
$\mathtt{t(?o,rdf:type,bro:ClosedPO) |\ ho]}$

\smallskip

$\mathtt{SELECT\ ?a\ ?p}$

\hspace*{6mm}
$\mathtt{WHERE\  output(?a,?i::bro:Purchase\_Order,?p)\ AND\ reachable(?a,?b,?p)\  AND}$ \\  
\hspace*{11mm}
$\mathtt{activity(?b::bro:Transportation)\ AND\ assigned(?b,?c::bro:Carrier,?p)}$ \\
\hspace*{11mm}
$\mathtt{AND\ [\ NOT\ EU\ (\ NOT\  en(?a,?p),\ en(?b,?p)\ )\ |\ ?p]}$ 

\subsubsection{BPKB Editor} 
This component provides a graphical user interface to define a BPKB and to interact with the BPAL Reasoner. A screen-shot of the main components of the GUI is depicted in Figure \ref{fig:gui}.
\begin{itemize}
\vspace{-1mm}
	\item The left panel (Figure \ref{fig:gui}.a) is the Package Explorer, providing a tree view of the resources available in the workspace, including BP schemas and ontologies.
\vspace{-1mm}
	\item  The central panel (Figure \ref{fig:gui}.b) is the BP Modeling View, based on the  STP BPMN Modeler\footnote{\url{http://www.eclipse.org/soa}}, comprising an editor and a set of tools to model business process diagrams using the BPMN notation.
\vspace{-1mm}
	\item On the left (Figure \ref{fig:gui}.c), the Ontology View allows for the visualization of OWL ontologies, published on the Internet or locally stored.
\vspace{-1mm}
	\item The bottom panel (Figure \ref{fig:gui}.d) is the Annotation View, an editor for the annotation of process elements with respect to the reference ontology.
\vspace{-1mm}
	\item The top-central panel (Figure \ref{fig:gui}.e) is the QuBPAL View, that provides a query prompt to access  the BPAL reasoner through the query mechanism. Results can be consulted in the result panel (Figure \ref{fig:gui}.f).
\end{itemize}

\begin{figure}
\centering
  \includegraphics[width=15cm]{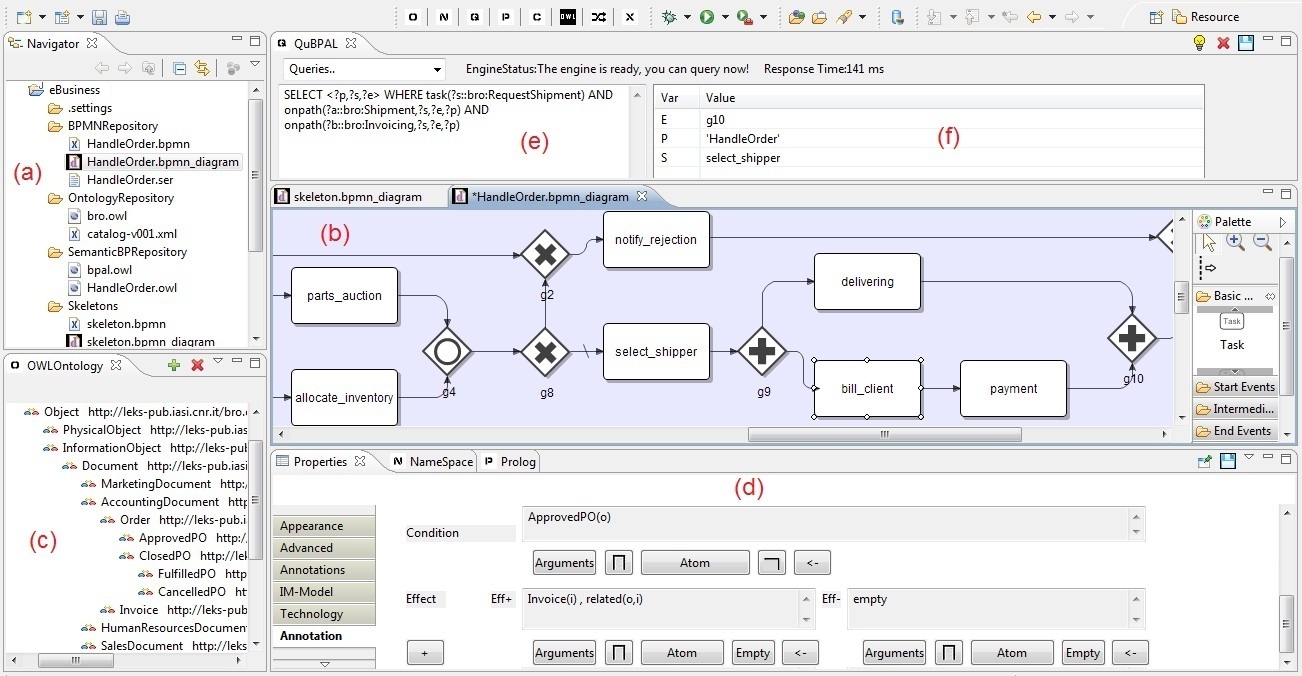}
	\caption{GUI of the BPAL platform}
	\label{fig:gui}
\end{figure}

\subsubsection{BPAL Reasoner} This component implements the reasoning methods described in Section \ref{section:reasoning}
by using the \textit{XSB Prolog}\footnote{\url{http://xsb.sourceforge.net/}} system \cite{xsb12}, which
is a Logic Programming system based on the SLG-resolution
inference strategy  recalled in Section \ref{subsec:lp}. 
As proved in Section \ref{sect:proof}, the tabling mechanism guarantees
the termination of query evaluation and the polynomial time
(in the size of the state space) verification of CTL properties.

Process schemas are imported into the BPKB from BPMN process models via
the \textit{BPMN2BPAL} interface. In order to ease the sharing and re-use of semantic meta-data, semantic information used and produced during the annotation process (i.e., reference ontologies and semantic annotations) can be exported and imported from OWL/RDF files by means of the \textit{RDF I/O} module. The underlying XSB Prolog implementation of the rule-based reasoner can deal  with either RDF, RDFS or OWL 2 RL ontologies. 
The \textit{BPKB Manager} handles the set-up and the interaction with the LP engine by  initializing and updating a BPKB. After populating the BPKB, inference is essentially performed by posing queries to the \textit{XSB Prolog} engine, connected through a Java/Prolog interface. To this end, the \textit{Query Manager} exposes functionalities to translate QuBPAL queries into LP queries, evaluate them, and collect the results in a textual form or export them in an XML serialization.

\subsection{Experimentation}
\label{sect:experimentation}
The approach has been applied to real-world scenarios coming from end-users involved in the European Project BIVEE\footnote{BIVEE: Business Innovation and Virtual Enterprise Environment (FoF-ICT-2011.7.3-285746)} and from the pilot conducted within a collaboration between the Italian CNR and SOGEI (ICT Company of the Italian  Ministry of Finance). The former is related to the modeling of production processes in manufacturing oriented networked-enterprises, while the latter regards the procedural modeling of legislative decrees in the tax domain. The experiments we have conducted are encouraging and revealed the practical usability of the tool and its acceptance by business experts. 

On a more technical side, the LP reasoner based on the XSB system shown a significant efficiency, since very sophisticated reasoning tasks have been performed on BPs of small-to-medium size (about one hundred of activities and several thousands of reachable states) in an acceptable amount of time and memory resources. Some empirical results are reported in the following, related to a dataset described in Table \ref{tab:datasets}.
We started by adapting a real world process, dealing with eProcurement, obtaining the BPS $P$, for which we report: the size, in terms of the number of flow elements; the number of reachable states; the number of exclusive, parallel, and inclusive gateways. As summarized in the table,  the considered BPS  does not contain logical errors (e.g., deadlocks) and is characterized by a considerable number of gateways, that is, branching/merging points (about 45\%  of the total number of elements). 
We then annotated in three different ways the process, obtaining $P_1,P_2,P_3$. For each one, in Table \ref{tab:datasets} we report: the number of reachable states; the coverage of the annotation, in terms of the  percentage of the annotated flow elements; the average size of each state, in terms of the number of ontological assertions (i.e., $t_f$ fluents) occurring in each state; the average size of the annotation, in terms of the number of $t_f$ fluents  occurring in the precondition/effect descriptions of the annotated flow elements; the errors exhibited by the BPS.  In particular, $P_1$ has been annotated without  preventing  logical errors induced by the annotation, $P_2$ presents a revised version of $P_1$ annotation, further extended in $P_3$.   

For the annotation of the BPS  we adapted an ontology covering documents and
production-related activities in the context of eProcurement and eBusiness, developed
within the BIVEE project, comprising about 100 concepts. 

\begin{table}[htbp!]

\begin{center}
\caption{Annotated processes used in the evaluation}
\label{tab:datasets}
\small
\begin{tabular}{|c|c|c|c|c|c|c|} \hline 
& \textit{Size} & \textit{States} & \textit{XOR} & \textit{PAR} & \textit{OR}& \textit{Errors}  \\ \hline 
\textbf{P} & 87 & 821 & 14 & 14 & 6 & No\\ \hline  
\end{tabular}
\begin{tabular}{|c|c|c|c|c|c|} \hline 
&  \textit{States} & \textit{Annotation} & \textit{Average}&  \textit{Average} & \textit{Errors}   \\  
&   & \textit{Coverage} & \textit{State Size}&  \textit{Annotation Size} &   \\  \hline

\textbf{P$_1$} & 944 & 35 \% & 7 & 3 & 2 non executable activities \\
&&&&& 150 inconsistent states \\ 
&&&&& 2 deadlocks \\ \hline 
\textbf{P$_2$} & 2172 & 70 \% & 11 &  5 & NO\\ \hline 
\textbf{P$_3$} & 3866 & 100 \% &  16 &  8 & NO\\ \hline 
\end{tabular}
\end{center}

\end{table}

The experiments have been performed on an Intel laptop, with a 3 GHz Core 4 CPU, 8 GB RAM and Windows operating system. For each BPS we first tested the set-up of the reasoner, which include the translation of the BPKB into LP rules, their loading into the XSB reasoner, and the computation of the state space, i.e., the transitive closure of the $result$ relation. Timing (measured in milliseconds) and memory occupation (measured in megabytes) are reported in Table \ref{tab:run-time}.  We then run the queries presented in Section \ref{section:verification} and the last presented in Section \ref{section:query}, representing respectively: the verification of the option to complete (Q1),  consistency condition (Q2), and executability (Q3) properties,  an exemplary compliance rule (Q4) and a retrieval query (Q5). For each query, the average timing obtained in 10 runs is reported.

\begin{table}[htbp!]
\begin{center}
\caption{Run-time phase evaluation}
\label{tab:run-time}
\small
\begin{tabular}{|c|c|c|c|c|c|c|c|} \hline 
 & \multicolumn{2}{|c|}{\textbf{\textit{State Space}}} & \multicolumn{5}{|c|}{\textbf{\textit{Query Evaluation}}}  \\ \hline
 & \textit{Time} & \textit{Memory } & \textit{Q1} & \textit{Q2} & \textit{Q3} & \textit{Q4} & \textit{Q5} \\ \hline 
 \textbf{$P$} & 265  & 35 & 60 & 100    &  60 & -  & -\\ \hline 
\textbf{$P_1$} & 1030  & 210 & 110 & 2710    & 110 & 50 & 30  \\ \hline 
\textbf{$P_2$} & 3300 &  670 & 530 & 4320  & 240   & 90 & 50\\ \hline 
\textbf{$P_3$} & 9720 & 1200 & 970 & 9250 & 405  & 105 & 60\\ \hline 
\end{tabular}
\end{center}
\end{table}

To better understand the performed tests, additional considerations are needed. Firstly, the above CTL queries have been executed \textit{after} the computation of the state space, which, due to the SLG-resolution strategy implemented by XSB, causes the population of the  \textit{tables} storing the intermediate results. The tables are then available in the subsequent queries, speeding up the computation. Secondly, to stress the engine, the evaluation of the performed queries requires the verification of ontology-based properties for each reachable state. Finally, the  amount of required memory depends on the strategy adopted by the engine for the management of the tables. In the above experiments the default behavior has been adopted   and, according to that, every intermediate result is materialized. This explains the large memory consumption, which, if needed, can be strongly reduced by introducing specific configurations to limit the use of tables, trading space for time. 

It is also worth noting that no code optimization has been performed, since the executed Prolog program is the direct translation of the  rules presented in this paper. Another remark regards the overhead introduced by the Java/Prolog bridge, which does not introduce a relevant performance degradation. 
Indeed, by running the same tests directly on XSB, without the Java infrastructure, the timings differ (up to a 10\%) only in the presence of a large amount of results, mainly due to the  inter-process data exchange. 

\section{Related Work}
\label{sect:rel_work}

\noindent \textbf{BP Modeling and Analysis.} Among several mathematical formalisms proposed for defining a
formal semantics of BP models, Petri nets \cite{vander-workflow} are the most used
paradigm to capture the execution semantics of graph-based
procedural languages (the BPMN case is discussed in \cite{petri-bpmn}). Petri net models enable  a large number of  techniques for the control flow analysis of processes, but they do not provide a suitable basis to represent and reason about additional domain knowledge.  In our framework we are able to capture the token game semantics underlying workflow models, and we can also declaratively  represent constructs, such as exception handling behavior or synchronization of active branches only (inclusive merge), which, due to their non-local semantics, are cumbersome to capture in standard Petri nets.
Furthermore, the logical grounding of our framework makes it easy to deal with the modeling of domain knowledge and the integration of  reasoning services.

Program analysis and verification techniques have been largely applied to the analysis of
process behavior, e.g., \cite{bpel-mc,compliance-ibm}. These papers are based on
the analysis of finite state models through model checking techniques \cite{clarke}, where
queries, formulated in some temporal logics, specify properties of process executions. However, these approaches are restricted to properties regarding the control flow only
(e.g., properties of the ordering, presence, or absence of tasks in process executions),
and severe limitations arise when ontology-related properties are included
as part of the model to be checked.

Other approaches based on Logic Programming that are worth mentioning are \cite{GiordanoMST13,kifer-rr-08,montali-vander}. \cite{GiordanoMST13} presents an approach to BP verification based on an extension of answer set programming  with temporal logic and constraints, where  the compliance of business rules is checked by bounded model checking techniques  extended with constraint solving for dealing with conditions on numeric data. \cite{kifer-rr-08,montali-vander} mainly focus on the analysis and on the enactment of flow models representing service choreographies, while we are not aware of  specific extensions that deal with the semantic annotation of procedural process models with respect to domain ontologies. 

\medskip

\noindent \textbf{Semantic Verification of BPs.} There is a growing body of contributions beyond pure control flow verification \cite{beyond-soundness,auditing-compliance,service-composition-semantics,semantic-mc}. In
\cite{beyond-soundness} the authors introduce the notion of Semantic Business Process
Validation, which  aims at verifying properties  related to the absence of logical errors which extend the notion of workflow soundness \cite{vander-workflow}. Validation is based on an execution semantics where token passing
control flow is combined with the AI notion of state change induced by domain-related
logical preconditions/effects. 
The main result is constituted by a validation  algorithm that runs in polynomial time in the size of the workflow graph, under some
restrictions on its structure and on the expressivity of the logic underlying the
domain axiomatization, i.e., binary Horn clauses.  This approach is  focused on providing efficient techniques for the
verification of specific properties, while the verification of arbitrary behavioral properties, such as the CTL formulae
allowed in our framework, is not addressed. Moreover, our language for annotations,
encompassing OWL 2 RL, is more expressive than binary Horn clauses. 
BP analysis techniques based on logical descriptions of effects of task execution are also proposed in   \cite{auditing-compliance,service-composition-semantics},  but they introduce algorithms in an informal way, since  a formal execution semantics is not provided, and a background ontology is not considered. 

In \cite{semantic-mc} the authors discuss a CTL model checking method for annotated state transition systems, encoding the procedural behavior of Web Services interactions. Given a query, in the form of a  CTL formula  containing conjunctive subqueries, a \textit{boolean} answer is computed in two steps: (1) a ground transition system is produced where each state contains all and only the description logic assertions \textit{relevant} to the input  query; (2)  the grounded model  is checked by a traditional propositional model checking algorithm. In contrast to our approach, the generation of the annotated transition system from a workflow model is neglected, and thus a semantics for activity preconditions/effects dealing also with the problems related to the state update is not given.  Furthermore, our framework allows much more expressive reasoning services, since it is not limited to the boolean verification of CTL queries. On the technical side, our approach avoids the burden of integrating several tools, since both the temporal and ontological reasoning are performed by the LP inference engine. One relevant advantage of the LP translation is the possibility of computing answers according to a pure top-down, goal-oriented strategy, which avoids the need of preliminary grounding the model and possibly performing a large number of inferences that are not necessary for answering  a given query.

Finally, we would like to mention a related research area, dealing with the verification of temporal properties in databases that evolve over time due to execution of actions operating on data (see \cite{data-bp-survey} for a survey). Recently, \cite{dejak-jair}  proposed \textit{Knowledge and Action Bases} (KABs),  where actions, encoded as  condition/action rules, modify the ABox of an ontology, encoded in a variant of the OWL 2 QL language. Under suitable restrictions, properties of KABs specified in the $\mu$-calculus are shown to be decidable, and their verification can be reduced to finite-state model checking. KABs describe systems that may reach an infinite number of states, unlike  our setting, where data are partially abstracted away\footnote{basically, the set of individuals in the ontology is bounded and fixed a-priori;  new values cannot be introduced during the enactment (e.g., by function terms)}, hence enforcing the reachable states to be a finite set.  However,  our framework is expressive enough to capture complex workflow specifications enriched with fluent expressions stated in terms of a background OWL 2 RL ontology. While the main goal of \cite{dejak-jair} is to provide theoretical results  that characterize the decidability and  (very high) complexity of KAB reasoning, our objective is more pragmatic and our formalization enables the implementation, through standard LP engines, of a wider set of (polynomial time) reasoning services, besides the verification of temporal properties. 
\medskip

\noindent \textbf{Process Ontologies.} The Process Specification Language (PSL) \cite{psl} is an ontology designed to formalize reasoning about processes in first-order logic. The basic structure that characterizes the behavior of a process in PSL is the \textit{occurrence tree} (whose model is inspired by the Situation Calculus \cite{reiter}), which contains all (infinite) sequences of occurrences of atomic activities starting from an initial state.  
Many extensions of PSL have been proposed to deal with time points,  objects, agents, and resources.  Although PSL is defined in first-order logic, which in principle makes behavioral specifications in PSL amenable to automated reasoning,  it is mostly intended as a means to facilitate correct and complete exchange of process information among manufacturing systems, rather then for computation. Indeed, it is  a very expressive framework whose associated reasoning tasks are intractable even for simple definitions, and undecidable  in general, due to the adoption of unrestricted first-order logic.  Furthermore, the systematic translation of procedural workflow descriptions into PSL has not been addressed, hence limiting its usability. 

Several papers proposed the extension to BP management of techniques developed in the context of the Semantic Web\footnote{See the work conducted within the SUPER project: \url{http://www.ip-super.org/}}. To this end several meta-model process ontologies have been proposed, with the aim of specifying in a declarative, formal, and explicit way the modeling constructs, and enabling the use of domain ontologies for the semantic reconciliation of model contents.  Some of them are derived from BP modeling notations (e.g.,  BPMN \cite{bpkb-iswc}), EPC \cite{sEPC},  XPDL \cite{oXPDL}, Petri nets \cite{oPetriNet}, while others have been designed in the context of interoperability, to overcome heterogeneities deriving from the adoption of different languages by mapping them to one common process ontology (e.g., GPO  \cite{gpo-thesis}, BPMO \cite{bpmo-tool}). The above approaches share some common features and goals: (1) they are based on standardized Web ontology languages; (2) they allow a machine-processable representation of BP models; (3) they enable query and search facilities; (4) they provide the means for relating  BP models to existing business dictionaries and background knowledge. While a BPAL BPKB provides all the above features, supporting OWL 2 RL for ontological modeling, it also integrates  behavioral modeling and a more expressive verification mechanism.

\medskip

\noindent \textbf{Semantic Web Services.} Another stream of related papers regards the semantic enrichment of Web Services, where relevant work  has been done within the OWL-S \cite{owl-s} and WSMO \cite{wsmo} initiatives. Both make an essential use of ontologies in order to facilitate the automation of discovering, combining and invoking electronic services over the Web. To this end they describe services  from two perspectives: from a \textit{functional} perspective a service is described in terms of its functionality, preconditions and effects, input and output; from a process \textit{perspective}, the service behavior is modeled as an orchestration of other services. However, in the above approaches the behavioral aspects are abstracted away,
thus hampering the availability of reasoning services related to the execution  of BPs.
To overcome such limitations, several solutions for the representation of service
compositions propose to translate the relevant aspects of the aforementioned service
ontologies into a more expressive language, such as first-order logic. 
Among them,  \cite{golog-owl-s} adopts the high-level agent
programming language Golog \cite{reiter}, \cite{swso,situation-calc-owl-s} rely on
Situation Calculus  variants, while \cite{fluent-ws-2010,fluent-ws-2013} are based on a direct translation of OWL-based service description into a Fluent Calculus theory. However, such approaches are mainly tailored to automated
service composition (i.e., finding a sequence of service invocations such that a given goal
is satisfied). Thus, the support provided for process definition, in terms of workflow
constructs, is very limited and they lack a clear mapping from standard
modeling notations. Furthermore, the adoption of a state-independent domain axiomatization (i.e., a DL TBOX) is not considered in the aforementioned approaches.  In contrast, our framework allows a much richer procedural description of processes, directly corresponding to BPMN diagrams.
Moreover, a reference ontology can  be used to ``enrich''
process descriptions by means of annotations written in OWL 2 RL,
one of the most widespread languages for ontology representation.

\section{Conclusions}

%

\subsection*{Summary}

In this paper we discussed a methodological framework and a technical solution  for the semantic enrichment of BP models, based on the synergic use of BPAL, a rule-based language adopted to provide a declarative representation of the procedural knowledge of a BP, and business ontologies, to capture the semantics of a business scenario. The resulting knowledge base provides a uniform and formal representation framework, suited for automated reasoning and equipped with a powerful inference mechanism supported by the programming systems developed in the area of Logic Programming. 

BPAL is a rule-based formalism for modeling the structure and the  behavior of a business process represented accordingly to a workflow perspective. It is essentially a process ontology, which provides a vocabulary, derived from BPMN, for specifying BPs,  and an explicit description of its meta-model and execution semantics  in terms of two core first-order logic theories which give formal definitions to the constructs of the language. In particular, from a control flow perspective, BPAL supports a  relevant fragment of the BPMN standard, allowing us to deal with a large class of process models.

We then proposed an approach for the semantic enrichment of BPs, where BPAL BP schemas are related through a semantic annotation to a  conceptualization of the business scenario formalized in a computational ontology. By integrating the rule-based ontology language OWL 2 RL with the structural and behavioral specification provided by BPAL, we are able to define a Business Process Knowledge Base (BPKB), as a collection of logical theories that provide a declarative representation of a repository of semantically enriched BPs. \\
\indent On top of this knowledge representation framework, we built a number of reasoning services which allow the user to  formulate complex queries that combine properties related to the structure, the behavioral semantics, and the ontological description of the BPs. We showed how advanced resolution strategies, such as the tabled resolution implemented in the XSB  Logic Programming system, guarantee a terminating, sound, and complete   evaluation of the queries that can be issued over a BPKB.  

\subsection*{Discussion}

The rule-based approach followed in our framework offers several advantages. First of all, it enables the
combination of the procedural and ontological perspectives in a very smooth and natural
way, thus providing a uniform framework for reasoning on properties that depend on the
sequence of operations that occur during process enactment and also on the domain where the process operates. Another advantage is the generality of the approach, which is open to further extensions,
since other knowledge representation applications can easily be integrated, by providing
a suitable translation to Logic Programming rules.

Furthermore, our approach does not introduce a new business process modeling language, but 
provides a framework where one can map and integrate knowledge
represented by means of existing formalisms. 
This is very important from a pragmatic point
of view, as one can express process-related knowledge by using standard modeling 
languages, while adding
extra reasoning services. We have adopted BPMN as a graphical modeling notation, and its XML linear form to import and manipulate BP models, possibly designed through external BP Management Systems. For what concerns the ontology representation, we have committed to OWL, the current de-facto standard for ontology modeling and meta-data exchange. In essence, we have proposed a progressive approach, tailored to enhanced adaptability, where a business expert can start with the (commercial) tool and notation of his/her choice, and then enrich its functionalities with the formal framework we provide.

Finally, since our rule-based representation can be
directly mapped to a class of logic programs, we can use standard Logic Programming
systems to perform reasoning tasks such as verification and querying through  a goal-oriented, efficient sound and complete evaluation procedure.

There are two main assumptions related to the practical applicability of our approach: the availability of ontologies and the willingness of an organization to describe their processes with semantic information. Clearly,  enabling additional reasoning services comes at the price of additional modeling efforts,   which may seriously hamper  the adoption of our solution; this is  a problem shared by many approaches based on Knowledge Representation techniques, in the Semantic Web related-research in particular. We now briefly discuss the impact of the above issues on the proposed approach.

The development of an ontology is a very complex task  that requires the expertise of knowledge engineers and domain experts, and hence, high costs.  Nevertheless, industrial products and services categorization standards, such as RosettaNet (\url{http://www.rosettanet.org/}) or eClass (\url{http://www.eclass-online.com/}), and  libraries of standard  business documents, such as UBL (\url{http://ubl.xml.org/}),  reflect some degree of community consensus, and can thus be valuable input for creating business domain ontologies \cite{eClassOwl}. Also the growing interest for  the publication of open data and their organization according to the Linked Data paradigm\footnote{http://linkeddata.org/} increase the availability of publicly accessible terminological resources.  Moreover, emerging methodologies for collaborative ontology building   may be adopted here to lift existing resources (e.g., glossaries, organizational and data models) into formal theories \cite{cts2013}. That said, it should be noticed that  our framework does not require a heavy-weight, richly axiomatized ontology to work.  The query capabilities can be still exploited  even in the presence of a thesaurus only, which defines a set of terms whose meaning is agreed upon, possibly arranged in hierarchical structures. In this case, the annotation is  reduced to tags taken from such a common glossary, but still retrieval and verification tasks with a practical relevance can  be performed. 

Also the semantic annotation is a time-consuming and error-prone task, which  does not pay off if a small number of BPs has to be managed. However, in situations where hundreds of process models are available within an organization, and many collaborations with other departments or companies take place, the alignment of the adopted terminology and the reasoning facilities enabled by the semantic annotation may create a significant added-value\footnote{See, e.g., the EU projects SUPER  (http://www.ip-super.org/), Plug-it (http://plug-it-project.eu), COIN (http://www.coin-ip.eu/) and BIVEE (http://www.bivee.eu/).}. Furthermore,  once the ontologies are available, the effort required to the user for creating annotations amounts to browsing and selecting ontology concepts (see Section \ref{sect:tool}). In addition, we do not require that every BP is fully annotated; in many situations only parts of the model may be of interest for specific querying or verification tasks. Finally,  approaches based on information retrieval and linguistic analysis can also be applied to support the annotation, suggesting correspondences between activity labels and terms defined in an ontology \cite{auto-ann2}. 

\subsection*{Future Work}

The results presented in this paper leave  several directions open for future research. First of all, we plan  to push forward the empirical investigation of the impact of our proposal in each application scenario we are addressing, as reported in Section \ref{sect:experimentation}.

On the technical level, a relevant aspect to be further elaborated  regards  the adoption of query optimization techniques to enhance the reasoning approach. As it stands, the reasoner    performs only simple optimizations based on the  re-ordering of literals, and  all the queries are evaluated with a pure goal-oriented, top-down approach, without any pre-processing of the knowledge base. We are confident that the query evaluation process can be strongly improved through more sophisticated query rewriting and program transformation techniques \cite{program-transformation}, which have been largely investigated in the area of  Logic Programming.   

We are also interested in applying the proposed framework  in other phases of the BP life-cycle. In particular, the trace semantics of BPAL appears a suitable starting point to support: (i) querying at run-time, i.e., performed over a running instance of the process during its enactment; (ii) a-posteriori, i.e., over the execution logs  of completed enactments,  by adopting Inductive Logic Programming techniques, such as the ones presented in \cite{miningILP}; (iii) verification techniques for BPs in the presence of data constraints, by following approaches based on Constraint Logic Programming such as, for instance, the one proposed in \cite{BPconstraintsCLP}.

Finally, we plan to extend the framework to also represent  the execution-level process knowledge, and support the transition between conceptual and executable processes  from a service-oriented perspective. That is, given a conceptual process model, Web services available in a repository are  selected  and possibly orchestrated to implement the process activities. The query-based support to process composition discussed in \cite{iesa} represents a first contribution in that direction.




\bibliographystyle{abbrv}
\bibliography{main-bib}

\end{document}